\documentclass[]{imsart}
\usepackage{amsthm,amsmath}
\usepackage[colorlinks,citecolor=blue,urlcolor=blue]{hyperref}

\usepackage{amsmath,amsfonts,amssymb,amsthm}

\usepackage{enumitem}

\RequirePackage{amsthm,amsmath,amsfonts,amssymb}
\RequirePackage[numbers]{natbib}
\RequirePackage[colorlinks,citecolor=blue,urlcolor=blue]{hyperref}
\RequirePackage{graphicx}
\usepackage{lineno,hyperref}

\startlocaldefs
\modulolinenumbers[5]

\usepackage[utf8]{inputenc}
\usepackage[english]{babel}

\usepackage[margin=1in]{geometry}
\usepackage{bm}
\usepackage{graphicx}
\usepackage{amssymb,amsmath,amsthm,mathrsfs}
\usepackage{epstopdf}
\usepackage{amsfonts,delarray}
\usepackage{algcompatible,amsmath}
\usepackage{rotating,color}
\usepackage{subfigure}
\usepackage{multirow}
\usepackage{titlesec,textcase}
\usepackage{longtable}
\usepackage{bbm}
\usepackage{rotating}
\usepackage{float}
\usepackage[ruled,vlined,linesnumbered]{algorithm2e}
\usepackage{etoolbox}
\SetKwInput{KwInput}{Input}                % Set the Input
\SetKwInput{KwOutput}{Output}              % set the Output
\newcounter{algoline}
\newcommand\Numberline{\refstepcounter{algoline}\nlset{\thealgoline}}
\AtBeginEnvironment{algorithm}{\setcounter{algoline}{0}}

\newtheorem{Theorem}{Theorem}[section]
\newtheorem{Lemma}[Theorem]{Lemma}

\newtheorem{Proposition}[Theorem]{Proposition}

\newtheorem{Definition}[Theorem]{Definition}

\newtheorem{theorem}{Theorem}
\makeatletter
\newcommand{\settheoremtag}[1]{% \settheoremtag{<tag>}
  \let\oldthetheorem\thetheorem% Store \thetheorem
  \renewcommand{\thetheorem}{#1}% Redefine it to a fixed value
  \g@addto@macro\endtheorem{% At \end{theorem}, ...
    \addtocounter{theorem}{-1}% ...restore theorem counter value and...
    \global\let\thetheorem\oldthetheorem}% ...restore \thetheorem
  }
\makeatother

\theoremstyle{definition}

\RequirePackage[normalem]{ulem} %DIF PREAMBLE
\definecolor{rp}{RGB}{83,54,106}

\def\boxit#1{\vbox{\hrule\hbox{\vrule\kern6pt\vbox{\kern6pt#1\kern6pt}\kern6pt\vrule}\hrule}}

\endlocaldefs

\begin{document}
\begin{frontmatter}
\title{Community detection in censored hypergraph}

\runtitle{Community detection in censored hypergraph}
%\thankstext{T1}{A sample additional note to the title.}
\runauthor{ }
\begin{aug}

\author[A]{\fnms{Mingao} \snm{Yuan}\ead[label=e1]{mingao.yuan@ndsu.edu}},
%\and
\author[B]{\fnms{Bin} \snm{Zhao}\ead[label=e2]{bin.zhao@ndsu.edu}}
\and
\author[C]{\fnms{Xiaofeng} \snm{Zhao}\ead[label=e3]{zxfstat@ncwu.edu.cn}}
%%%%%%%%%%%%%%%%%%%%%%%%%%%%%%%%%%%%%%%%%%%%%%
%% Addresses                                %%
%%%%%%%%%%%%%%%%%%%%%%%%%%%%%%%%%%%%%%%%%%%%%%
\address[A]{Department of Statistics,
North Dakota State University, USA,
\printead{e1}}

\address[B]{Department of Statistics,
North Dakota State University, USA,
\printead{e2}}

\address[C]{School of Mathematics and Statistics,
North China University of Water Resources and Electric Power, China,
\printead{e3}}
\end{aug}

\begin{abstract}
Community detection refers to the problem of clustering the nodes of a network (either graph or hypergrah) into groups. Various algorithms are available for community detection  and all these methods apply to uncensored networks. In practice, a network may has censored (or missing) values and it is shown that censored values have non-negligible effect on the structural properties of a network. In this paper, we study community detection in censored $m$-uniform hypergraph from information-theoretic point of view. We derive the information-theoretic threshold for exact recovery of the community structure. Besides, we propose a polynomial-time algorithm to exactly recover the community structure up to the threshold. The proposed algorithm consists of a spectral algorithm plus a refinement step. It is also interesting to study whether a single spectral algorithm without refinement achieves the threshold. To this end, we also explore the semi-definite relaxation algorithm and analyze its performance. 

\end{abstract}



\end{frontmatter}

\section{Introduction}
\label{S:1}
Many complex data sets can be modelled as a network of items (nodes). One of the most popular topic in network data mining is to understand which items are similar to each other. Community detection refers to the problem of clustering the nodes of network into groups based on similarity. Community detection
is widely used in analysis of social networks \cite{GZFA,ZLZ11}, protein-to-protein interaction networks \cite{CY06}, image segmentation \cite{SM97} and so on. Existing literature in community detection can be roughly classified into two categories: (1) derive information-theoretic threshold for recovering the community structure\cite{ABH16,MNS15,MNS17,CLW18,DGMS21,HWX18,YS21}; (2) devise efficient algorithms to recover the community structure \cite{GD14,GD17,LZ20,LJY15,KSX20,YB21,ALS18,ALS19,HWX16,GMZZ16,WF21,ZW21,LR15,J15}. See \cite{A18,BTYZQ21} for more references. All these methods apply to uncensored networks.

In practice, a network data may have censored or missing values. For example,  in social network, non-response of actors can cause missingness of ties \cite{H09,GH16}; in MRI network, missingness may be due to the high cost involved with PET scanning \cite{LGYS18}. Missing values have non-negligible effects on the structural properties of a network \cite{H09,SMM18}. Most existing algorithms for community detection apply to uncensored networks. A natural question is how to recover communities in a censored network. As far as we know, \cite{ABBS14} is the first to deal with community detection in censored graph and obtains the information-theoretic threshold for exact recovery of communities. Recently, \cite{DGMS21} shows spectral algorithm without refinement step can exactly recover the community structure in censored graph up to the information-theoretic threshold.

In this paper, we consider community detection in censored hypergraph from information-theoretic point of view. More specifically, we derive the information-theoretic threshold for exact recovery. Interestingly, the threshold is generally larger than that in the graph case. In this sense, community detection in censored hypergraph is harder than in censored graph. Besides, we propose a polynomial-time algorithm that can exactly recover the community structure up to the information-theoretic threshold. The proposed algorithm consists of a spectral algorithm plus a refinement step. It is also interesting to study whether a single spectral algorithm without refinement can achieve the threshold as in the censored graph case \cite{DGMS21}. To this end, we study the semi-definite relaxation algorithm and provide a sufficient condition for the algorithm to achieve exact recovery.

\subsection{The censored hypergraph block model}

For a positive integer $n$,
let $\mathcal{V}=\{1,2,\dots,n\}$ denote a set of nodes and $\mathcal{E}$ be a set of subsets of $\mathcal{V}$. The pair $\mathcal{H}_m=(\mathcal{V},\mathcal{E})$ is called an \textit{undirected} $m$-uniform hypergraph if $|e|=m$ for every $e\in\mathcal{E}$. That is, each element $e\in\mathcal{E}$ (called hyperedge) contains exactly $m$ distinct nodes. The hypergraph
$\mathcal{H}_m$ can be represented as
a $m$-dimensional symmetric array $A=(A_{i_1,\ldots,i_m})\in\{0,1\}^{\otimes n^m}$, where $A_{i_1i_2\dots i_m}=1$ if $\{i_1,i_2,\dots,i_m\}$ is a hyperedge and $A_{i_1i_2\dots i_m}=0$ otherwise. Besides, $A_{i_1i_2\dots i_m}=A_{j_1j_2\dots j_m}$ if $\{i_1,i_2,\dots, i_m\}=\{j_1,j_2,\dots, j_m\}$. In this paper, self-loop is not allowed, that is, $A_{i_1i_2\dots i_m}=0$ if $|\{i_1,i_2,\dots, i_m\}|< m$. When $m=2$, $\mathcal{H}_2$ is just the usual graph that has been widely used in community detection problems \cite{A18}. A hypergraph is said to be random if elements of the adjacency tensor are random. Throughout this paper, we focus on hypergraph genereated from the Censored $m$-uniform Hypergraph Stochastic Block Model (CHSBM) $\mathcal{H}_{m}(n,p,q,\alpha)$ defined below. 
\begin{Definition}[Censored $m$-uniform Hypergraph Stochastic Block Model (CHSBM)]
Each node $i\in \mathcal{V}$ is randomly and independently assigned a label $\sigma_i$ with
\[
\mathbb{P}(\sigma_i=+1)=\mathbb{P}(\sigma_i=-1)=\frac{1}{2}.\]
Let $\sigma=(\sigma_1,\dots,\sigma_n)$ be a vector of labels, $\textrm{I}_+(\sigma)=\{i|\sigma_i=+1\}$ and $\textrm{I}_-(\sigma)=\{i|\sigma_i=-1\}$.
The nodes in $\textrm{I}_+(\sigma)$ and $\textrm{I}_-(\sigma)$ constitute two communities. The distinct nodes $i_1,i_2,\dots,i_m$ form a hyperedge with probability $p$ if $\{i_1,i_2,\dots,i_m\}$ is a subset of $\textrm{I}_+(\sigma)$ or $\textrm{I}_-(\sigma)$ and $q$ otherwise. Each hyperedge status is revealed independently with probability $\alpha$. The hyperedge of the resulting hypergraph takes value in $\{1,0,*\}$, where $*$ means a hyperedge is censored or missing (the hyperedge status is not revealed). This model is denoted as $\mathcal{H}_{m}(n,p,q,\alpha)$.
\end{Definition}

Each hyperedge in  $\mathcal{H}_{m}(n,p,q,\alpha)$ with $\alpha<1$ has three status: 1 (present), 0 (absent) or $*$ (censored or missing). When $\alpha=1$, the hypergraph is uncensored and $\mathcal{H}_{m}(n,p,q,1)$ is just the usual hypergraph stochastic block model  \cite{GD14,GD17,CLW18,KBG18,KSX20,YS21}. The Censored Stochastic Block Model $CSBM(p,q,\alpha)$ studied in \cite{DGMS21} corresponds to  $\mathcal{H}_{2}(n,p,q,\alpha)$. Throughout this paper, we assume $p,q\in(0,1)$ are fixed constants, $p>q$ and  $\alpha=\frac{t\log n}{n^{m-1}}$ for some constant $t>0$. 

\subsection{Summary of main result}

Given a hypergraph $A$ generated from $\mathcal{H}_{m}(n,p,q,\alpha)$, community detection refers to the problem of recovering the unknown true label vector $\sigma$, or equivalently, identifying the sets $I_+(\sigma)$ and $I_-(\sigma)$. We say an estimator $\hat{\sigma}$ is an exact recovery of $\sigma$ or $\hat{\sigma}$ exactly recovers $\sigma$ or $\hat{\sigma}$ achieves exact recovery if 
\[\mathbb{P}(\exists s\in\{\pm1\}:\hat{\sigma}=s\sigma)=1-o(1).\]
That is, the estimator $\hat{\sigma}$ is equal to $\sigma$ or $-\sigma$ with probability $1-o(1)$.
If there exists an estimator $\hat{\sigma}$ that exactly recovers $\sigma$, we say exact recovery is possible. Otherwise, we say exact recovery is impossible. 

For $m=2$, \cite{DGMS21} establishes the sharp information-theoretic threshold for exact recovery. The authors show that spectral algorithm can exactly recover the true label without refinement step. It is not immediately clear how $m\geq3$ changes the threshold for exact recovery. More importantly, the spectral method in \cite{DGMS21} can not be straightforwardly extended to $m\geq3$, since the spectral analysis of tensor is still not well developed. 

In this paper, we focus on $m\geq3$ and derive the sharp information-theoretic threshold for exact recovery.  Define   $I_m(p,q)$ as
\begin{equation*}\label{snr}
    I_m(p,q)=\frac{2^{m-1}(m-1)!}{(\sqrt{p}-\sqrt{q})^{2}+(\sqrt{1-p}-\sqrt{1-q})^{2}}.
\end{equation*}
Theorem \ref{MLEimpossible} shows that the maximum likelihood estimator (MLE) does not coincide with the true label with probability $1-o(1)$ if $t<I_m(p,q)$. Theorem \ref{MLEpossible} says that MLE succeeds with probability $1-o(1)$ if $t>I_m(p,q)$. For efficient algorithms, we propose a spectral algorithm plus refinement step that can achieve exact recovery up to the information-theoretic threshold, see Theorem \ref{polytime}. 
Finally, we prove in Theorem \ref{Spectra} that the semidefinite relaxation algorithm can exactly recover the true label under mild conditions.
The following Table \ref{region} summarizes our main results. For $m=2,3$ and $q=0.2$, Figure \ref{phase} displays the region in which exact recovery is impossible (red region) and the region where exact recovery is possible (green region). Interestingly, with fixed $q$, the red region of $m=3$ contains that of $m=2$ as a proper subset. In this sense, exact recovery gets harder as $m$ increases. For fixed $q,m$, 
$I_m(p,q)$ decreases as $p$ goes to one, hence exact recovery becomes easier.  

\begin{table}[H]
\vspace{-2mm}
	\centering
		\caption{\it Regions for exact recovery.}
	\begin{tabular}{|p{3cm} p{4cm} |}
	\hline
	Region & Exact Recovery  \\
	\hline
	(a) $t<I_m(p,q)$ &  	 Exact recovery is impossible        \\
	(b) $t>I_m(p,q)$   & Exact recovery is possible        \\
	\hline
\end{tabular}
\label{region}
\vspace{-5mm}
\end{table}

\begin{figure}[h] 
%\vspace{3mm}
\centering
\includegraphics[width=8cm,height=8cm]{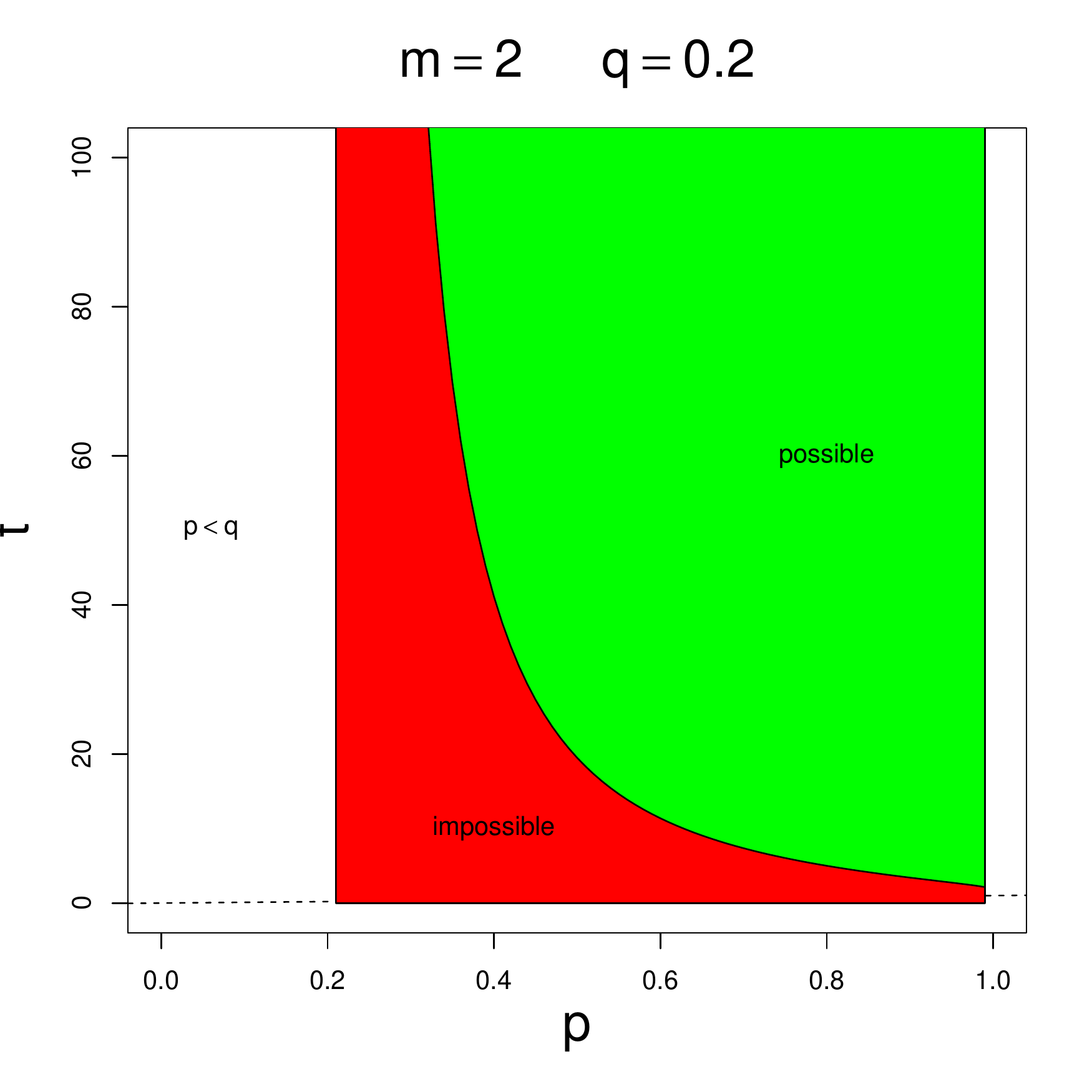}
\vspace{-5mm}
\includegraphics[width=8cm,height=8cm]{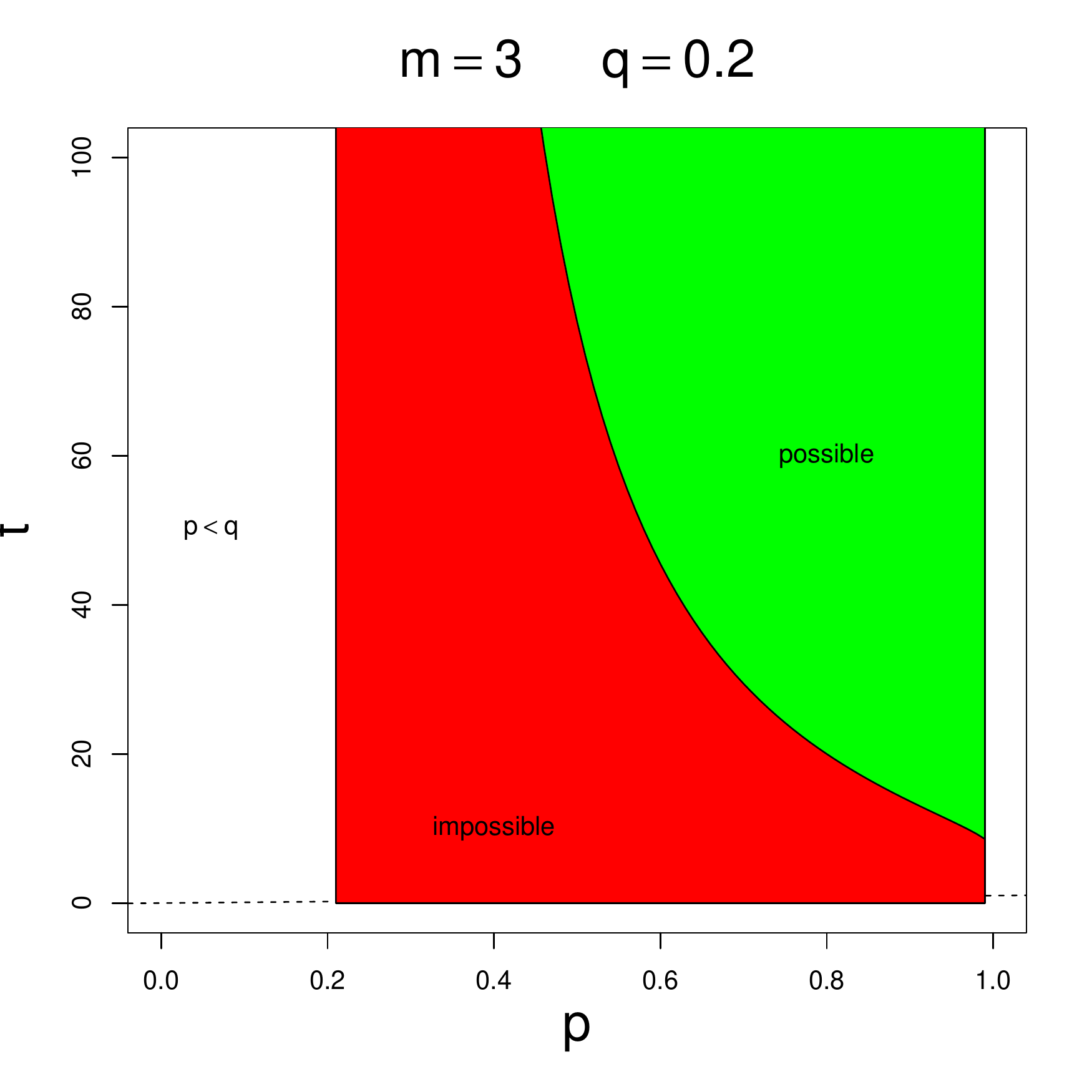}
\caption{\it\small Regions for exact recovery with $m=2,3$ and $q=0.2$. Red: exact recovery is impossible. Green: exact recovery is possible. }
\label{phase}
\end{figure}

Throughout this paper, we adopt the Bachmann-Landau notation $o(1),O(1)$. For two positive sequences $a_n,b_n$, we write $a_n\sim b_n$ if $\lim_{n\rightarrow\infty} \frac{a_n}{b_n}=1$. Denote $a_n\asymp b_n$ if $ 0<c_1\leq\frac{a_n}{b_n}\leq c_2<\infty$ for constants $c_1,c_2$. Denote $a_n\gg b_n$ or $b_n\ll a_n$ if $\lim_{n\rightarrow\infty}\frac{a_n}{b_n}=\infty$. For a square matrix $M$, $||M||$ denotes the operator norm of $M$ and $M \succeq0$ means $M$ is symmetric and positive semidefinite. Define $\langle M,N\rangle=\sum_{i,j}M_{ij}N_{ij}$.

\section{Main Results}\label{subgraph}

In this section, we present information theoretic threshold for exact recovery on the censored hypergraph stochastic block model. Firstly, we use the maximum likelihood method to show that exact recovery is impossible if $t<I_m(p,q)$. Then we prove that MLE can exactly recover the true label if $t>I_m(p,q)$. Combining these two results yields the sharp information-theoretic threshold for exact recovery. This threshold provides a benchmark for developing practical recovery algorithms. Since the time complexity of MLE is not polynomial in $n$, we propose a polynomial-time algorithm that achieves exact recovery if $t>I_m(p,q)$.

\subsection{Sharp threshold for exact recovery}\label{MLE}

In this subsection, we derive a sharp phase transition threshold for exact recovery. The first result specifies a sufficient condition for impossibility of exact recovery. 
\begin{Theorem}\label{MLEimpossible}
For each fixed integer $m\geq2$, if $t<I_m(p,q)$, then $\mathbb{P}(\hat{\sigma}=\sigma)=o(1)$ for any estimator $\hat{\sigma}$. Here $I_m(p,q)$ is defined as 
\begin{equation}\label{snr}
    I_m(p,q)=\frac{2^{m-1}(m-1)!}{(\sqrt{p}-\sqrt{q})^{2}+(\sqrt{1-p}-\sqrt{1-q})^{2}}.
\end{equation}
\end{Theorem}

Theorem \ref{MLEimpossible} states that any estimator can not exactly recover the true label if $t<I_m(p,q)$. For $m=2$, $I_2(p,q)$ is just $t_c(p,q)$ in \cite{DGMS21}. Our result can be considered as a nontrivial extension of Theorem 2.1 in \cite{DGMS21}. Interestingly, with fixed $p,q$, the region $t<I_2(p,q)$ is smaller than $t<I_m(p,q)$ for $m\geq3$. Similar phenomenon exists in exact recovery of community in uncensored hypergraph stochastc block model \cite{KBG18}. However, this phenomenon significantly differs from that in hypothesis testing for communities. For example, \cite{YS21b} derived the sharp boundary for testing the presence of a dense subhypergraph. When the number of nodes in the dense subhypergraph is not too small, the region where any test is asymptotically powerless for $m=2$ is larger than $m\geq3$.

The next result shows that the threshold $I_m(p,q)$ is actually sharp for exact recovery.

\begin{Theorem}\label{MLEpossible}
For each fixed integer $m\geq2$, if $t>I_m(p,q)$ with $I_m(p,q)$ defined in (\ref{snr}), then the maximum likelihood estimator exactly recovers the true label with probability $1-o(1)$. 
\end{Theorem}

By Theorem \ref{MLEpossible}, if $t>I_m(p,q)$, the true label can be exactly recovered by the maximum likelihood estimator estimator. Combining Theorem \ref{MLEimpossible} and Theorem \ref{MLEpossible}, we get the sharp boundary $t=I_m(p,q)$ for exact recovery, which is a surface in $\mathbb{R}^3$. For illustration, we visualize the regions $t>I_m(p,q)$ and $t<I_m(p,q)$ with $q=0.2$ and $m=2,3$ in Figure \ref{phase}. The red region represents $t<I_m(p,0.2)$ where exact recovery is impossible. The green region corresponds to $t>I_m(p,0.2)$ where exact recovery is possible. Clearly, the green region for $m=3$ is smaller than $m=2$. In this sense, exact recovery gets harder as $m$ increases. 

\subsection{Efficient algorithm for exact recovery}\label{twostage}
 
Since the time complexity of MLE is not polynomial in $n$, we propose an efficient algorithm to reconstruct the two communities up to the information theoretic threshold. The algorithm starts with a random splitting of the hypergraph $A$ into two parts. Then a spectral algorithm is applied to the first part, followed by a refinement based on the second part. We describe the algorithm in the following three steps.

Firstly, we randomly split the hypergraph $A$ into two parts. Denote $M_m=\{(i_1,i_2,\dots,i_m)\mid 1\leq i_1<\dots <i_m\leq n\}$.
Let $S_1$ be a random subset of $M_m$ obtained by including each element of $M_m$ in $S_1$ with probability $\frac{\log \log n}{\log n}.$ Let $S_2$ be the compliment of $S_1$ in $M_m$, that is, $S_2=M_m-S_1$.
Define a hypergraph $\Tilde{A}$ as
\begin{equation*}
\Tilde{A}_{i_1i_2\dots i_m}=\left\{
\begin{array}{cc}
    \mathbbm{1}{[A_{i_1i_2\dots i_m}=1]}, & \{i_1,i_2,\dots,i_m\}\in S_1, \\
    0, & \textrm{otherwise}.
\end{array}
\right.
\end{equation*}
Here $\mathbbm{1}{[E]}$ is the indicator function of event $E$.
Define hypergraph $\bar{A}$ as
\begin{equation*}
\bar{A}_{i_1i_2\dots i_m}=\left\{
\begin{array}{cc}
    A_{i_1i_2\dots i_m}, & \{i_1,i_2,\dots,i_m\}\in S_2, \\
    *, & \textrm{otherwise}.
\end{array}
\right.
\end{equation*}
Then hypergraph $A$ is randomly divided into two independent hypergraphs $\Tilde{A}$ and $\bar{A}$.

In the second step, we apply the weak recovery algorithm in \cite{ALS18} to $\Tilde{A}$ and obtain two communities, denoted as $\Tilde{I}_+(\sigma)$, and $\Tilde{I}_-(\sigma)$. With probability $1-o(1)$, $n-o(n)$ of the nodes are correctly labelled.

The last step is to refine the communities $\Tilde{I}_+(\sigma)$ and $\Tilde{I}_-(\sigma)$ based on $\bar{A}$.
For a set $S\subset [n]$, define $e(i, S)$ as
\[e(i, S)=\sum\limits_{\substack{i_2,\dots ,i_m\in S\setminus \{i\}\\ i_2<\dots<i_m}}\left(\log \left(\frac{p}{q}\right)\mathbbm{1}{[\bar{A}_{ii_2\dots i_m}=1]}+\log \left(\frac{1-p}{1-q}\right)\mathbbm{1}{[\bar{A}_{i_1i_2\dots i_m}=0]}\right).\]
For each node $i\in \Tilde{I}_+(\sigma)$, flip the label of $i$ if 
\[e(i, \Tilde{I}_+(\sigma))<e(i, \Tilde{I}_-(\sigma)).\]
For each node $j\in \Tilde{I}_-(\sigma)$, flip the label of $j$ if
\[e(j, \Tilde{I}_-(\sigma))<e(j, \Tilde{I}_-(\sigma)).\]
Let $\hat{I}_+(\sigma)$ and $\hat{I}_-(\sigma)$ be the resulting communities. If $|\hat{I}_+(\sigma)|\neq |\tilde{I}_+(\sigma)|$, output $\tilde{I}_+(\sigma)$ and $\tilde{I}_-(\sigma)$; otherwise output $\hat{I}_+(\sigma)$ and $\hat{I}_-(\sigma)$.

The above algorithm is summarized in Algorithm 1.

\begin{Theorem}\label{polytime}
For each fixed integer $m\geq2$, if $t>I_m(p,q)$ with $I_m(p,q)$ defined in (\ref{snr}), then Algorithm 1 exactly recovers the true label with probability $1-o(1)$. 
\end{Theorem}

Theorem \ref{polytime} states that the information theoretic threshold can be attained by an algorithm with polynomial time complexity.
\vskip 5mm

\begingroup
\LinesNumberedHidden
\begin{algorithm}[H]
  \SetNlSty{textbf}{}{:}
  \setcounter{AlgoLine}{-1}
  \Numberline \KwInput{A censored $m$-uniform hypergraph $A$ generated from $\mathcal{H}_{m}(n,p,q,\alpha)$.}
  \Numberline \textbf{Step1: random splitting.} \\ Randomly select elements in $M_m=\{(i_1,i_2,\dots,i_m)\mid 1\leq i_1<\dots <i_m\leq n\}$ with probability $\frac{\log \log n}{\log n}$ to form a subset $S_1\subset M_m$ and let $S_2=M_m-S_1$. Construct hypergraph $\tilde{A}$ as $\Tilde{A}_{i_1i_2\dots i_m}=\mathbbm{1}{[A_{i_1i_2\dots i_m}=1]}$, if ${i_1,i_2,\dots,i_m}\in S_1$ and $\Tilde{A}_{i_1i_2\dots i_m}=0$ otherwise. Construct hypergraph $\bar{A}$ as $\bar{A}_{i_1i_2\dots i_m}=A_{i_1i_2\dots i_m}$ if ${i_1,i_2,\dots,i_m}\in S_2$ and $\bar{A}_{i_1i_2\dots i_m}=*$ otherwise. \\
  \Numberline \textbf{Step 2: spectral algorithm.}
  \caption{Spectral algorithm plus refinement for exact recovery} Apply the weak recovery algorithm in \cite{ALS18} to $\Tilde{A}$ and denote the community output as $\Tilde{I}_+(\sigma)$, $\Tilde{I}_-(\sigma)$. \\
  \Numberline \textbf{Step 3: refinement.} \\
  Flip the label of  $i\in\Tilde{I}_+(\sigma)$ if 
  $e(i, \Tilde{I}_+(\sigma))<e(i, \Tilde{I}_-(\sigma)).$\\ Flip the label of $j\in\Tilde{I}_-(\sigma)$ if 
  $e(j, \Tilde{I}_-(\sigma))<e(j, \Tilde{I}_-(\sigma)).$
  \\
  Let $\hat{I}_+(\sigma)$ and $\hat{I}_-(\sigma)$ be the resulting communities. \\
  \Numberline \KwOutput{ If $|\hat{I}_+(\sigma)|\neq |\tilde{I}_+(\sigma)|$, output $\tilde{I}_+(\sigma)$ and $\tilde{I}_-(\sigma)$; otherwise output $\hat{I}_+(\sigma)$ and $\hat{I}_-(\sigma)$.}
\end{algorithm}
\endgroup

 \subsection{Semidefinite relaxation algorithm}
 
 In subsection \ref{twostage}, we show that spectral algorithm with refinement step can achieve\ exact recovery.
It is also interesting to study whether a single spectral algorithm without refinement achieves the threshold. In graph case ($m=2$), the answer is confirmative and the semidefinite relaxation algorithm and spectral algorithm are shown to succeed without refinement step \cite{HWX16,DGMS21}. In hyergraph case ($m\geq3$), either censored or uncensored, it is still an open problem. In this subsection, we study the semidefinite relaxation algorithm and analyze its performance.  To this end, we define a new hypergraph based on the given hypergraph $A$ and transform it to a weighted graph. Then we show the semidefinite relaxation algorithm applied to the weighted graph can achieve exact recovery. 

Define hypergraph $\widetilde{A}$ based on $A$ as
\[\widetilde{A}_{i_1i_2\dots i_m}=\mathbbm{1}{[A_{i_1i_2\dots i_m}=1]},\] 
and $\widetilde{A}_{i_1i_2\dots i_m}=0$ if $|\{i_1,i_2,\dots ,i_m\}|\leq m-1$. Each hyperedge $\widetilde{A}_{i_1i_2\dots i_m}$ takes value in $\{1,0\}$.
The hypergraph $\widetilde{A}$ shares the same community structure as $A$, since
\begin{equation}\nonumber
\mathbb{E}(\widetilde{A}_{i_1i_2\dots i_m})=
\begin{cases}
p\alpha,&  \{i_1,i_2,\dots ,i_m\}\subset I_+(\sigma)\ or\ I_-(\sigma);\\
q\alpha,& \textrm{otherwise.} 
\end{cases}
\end{equation}
Next, we construct a weighted graph $G=[G_{ij}]$ based on $\widetilde{A}$ by
\[G_{ij}=\sum_{1\leq i_3<\dots<i_m\leq n}\widetilde{A}_{iji_3\dots i_m}.\]
Define the semidefinite program problem (SDP) as
\begin{equation}\label{semi}
\begin{array}{rrclcl}
\displaystyle \max_{Y} & \multicolumn{3}{l}{\langle G,Y\rangle}\\
\textrm{s.t.} & Y \succeq0\\
&\langle Y,J\rangle=0  \\
& Y_{ii}=0,\ i\in [n],
\end{array}
\end{equation}
where $J$ is $n\times n$ all-one matrix.
Suppose $\sigma$ is the true label and denote $Y=\sigma^T\sigma$. Let $\widehat{Y}$ be the solution to semidefinite program problem (\ref{semi}). The following result provides a sufficient condition under which  $\widehat{Y}$ is an exact recovery of $Y$.
\begin{Theorem}\label{Spectra}
For each fixed integer $m\geq2$, let 
\[J_m(p,q)=\frac{2^{m+2}(m-2)![mp-(m-2^{m})q] }{(p-q)^2}.\]
If $t>J_m(p,q)$, then $\mathbb{P}(\widehat{Y}=Y)=1-o(1)$, where $Y=\sigma^T\sigma$ with true label $\sigma$.
\end{Theorem}

 Note that $J_m(p,q)>I_m(p,q)$ for each $m\geq2$. When $m=2$ and the graph is uncensored, $\widehat{Y}$ can exactly recover the true label up to the information theoretic threshold \cite{HWX16}.  
However, for $m\geq3$, it is unclear whether $\widehat{Y}$ succeeds or not in the range $I_m(p,q)<t<J_m(p,q)$. Similar gap exists in uncencored hypergraph case \cite{KBG18}.

\section{Proof of Theorem \ref{MLEimpossible}}\label{fiv}

In this section, we provide proof of Theorem \ref{MLEimpossible}. 

\textit{Proof of Theorem \ref{MLEimpossible} :}
Let $l(\sigma)$ be the log-likelihood function of a label $\sigma$. By Proposition 4.1 in \cite{DGMS21}, if there are labels $\eta_t$ ($1\leq t\leq k_n$) with $k_n\rightarrow\infty$ such that $l(\eta_1)=l(\eta_2)=\dots=l(\eta_{k_n})$, then 
the maximum likelihood estimator (MLE) fails to exactly recover the true label with probability $1-o(1)$. Our proof proceeds by constructing labels $\eta_t$ ($1\leq t\leq k_n$) with $k_n\rightarrow\infty$ under the condition $t<I_m(p,q)$.

Firstly, we write down the explicit expression of the likelihood function. Note that for distinct nodes $i_1,i_2,\dots,i_m$, we have
\begin{equation*}
A_{i_1i_2\dots i_m}=\left\{
\begin{array}{cc}
    1 & , \\
    0 & , \\
    * &.
\end{array}
\right.
\end{equation*}
For convenience, let $\mathbbm{1}[E]$ be the indicator function of event $E$ and
\[\mathbbm{1}_{i_1i_2\dots i_m}(\sigma)=\mathbbm{1}{[\sigma_{i_1}=\sigma_{i_2}=\dots=\sigma_{i_m}]}.\]
Then the likelihood function for $\sigma$ given an observation of hypergraph $A$ from $\mathcal{H}_{m}(n,p,q,\alpha)$ is
\begin{eqnarray*}
L&=&\prod_{1\leq i_1<\dots <i_m\leq n}(p\alpha)^{\mathbbm{1}{[A_{i_1i_2\dots i_m}=1]}\mathbbm{1}_{i_1i_2\dots i_m}(\sigma)}[\alpha(1-p)]^{\mathbbm{1}{[A_{i_1i_2\dots i_m}=0]}\mathbbm{1}_{i_1i_2\dots i_m}(\sigma)} \\
&&\times (q\alpha)^{\mathbbm{1}{[A_{i_1i_2\dots i_m}=1]}(1-\mathbbm{1}_{i_1i_2\dots i_m}(\sigma))}[\alpha(1-q)]^{\mathbbm{1}{[A_{i_1i_2\dots i_m}=0]}(1-\mathbbm{1}_{i_1i_2\dots i_m}(\sigma))}(1-\alpha)^{\mathbbm{1}{[A_{i_1i_2\dots i_m}=*]}} \\
&=&\prod_{1\leq i_1<\dots <i_m\leq n}(1-\alpha)^{\mathbbm{1}{[A_{i_1i_2\dots i_m}=*]}}(q\alpha)^{\mathbbm{1}[A_{i_1i_2\dots i_m}=1]}\left(\frac{p}{q}\right)^{\mathbbm{1}{[A_{i_1i_2\dots i_m}=1]}\mathbbm{1}_{i_1i_2\dots i_m}(\sigma)} \\
&&\times [\alpha(1-q)]^{\mathbbm{1}{[A_{i_1i_2\dots i_m}=0]}}\left(\frac{1-p}{1-q}\right)^{\mathbbm{1}{[A_{i_1i_2\dots i_m}=0]}\mathbbm{1}_{i_1i_2\dots i_m}(\sigma)} \\
&=&\prod_{1\leq i_1<\dots <i_m\leq n}(1-\alpha)^{\mathbbm{1}{[A_{i_1i_2\dots i_m}=*]}}(q\alpha)^{\mathbbm{1}[A_{i_1i_2\dots i_m}=1]}[\alpha(1-q)]^{\mathbbm{1}{[A_{i_1i_2\dots i_m}=0]}} \\
&&\times \prod_{1\leq i_1<\dots <i_m\leq n}\left(\frac{p}{q}\right)^{\mathbbm{1}{[A_{i_1i_2\dots i_m}=1]}\mathbbm{1}_{i_1i_2\dots i_m}(\sigma)}\left(\frac{1-p}{1-q}\right)^{\mathbbm{1}{[A_{i_1i_2\dots i_m}=0]}\mathbbm{1}_{i_1i_2\dots i_m}(\sigma)}.
\end{eqnarray*}
The maximum likelihood estimator(MLE) is obtained by maximizing $L$ with respect to $\sigma$. The first product factor of $L$ does not involve $\sigma$.  Hence we only need to maximize the second product factor of $L$ to get MLE. Denote
\begin{eqnarray*}
l(\sigma)&=&\sum\limits_{1\leq i_1<\dots <i_m\leq n}\left[\log \left(\frac{p}{q}\right)\mathbbm{1}{[A_{i_1i_2\dots i_m}=1]}\mathbbm{1}_{i_1i_2\dots i_m}(\sigma)+\log \left(\frac{1-p}{1-q}\right)\mathbbm{1}{[A_{i_1i_2\dots i_m}=0]}\mathbbm{1}_{i_1i_2\dots i_m}(\sigma)\right].
\end{eqnarray*}
 The log-likelihood function is equal to
\begin{equation}\label{loglike}
\log L=R_{n}+l(\sigma),
\end{equation}
where $R_n$ is independent of $\sigma$.

Below we construct labels $\eta_t$ ($1\leq t\leq k_n$) with $k_n\rightarrow\infty$ under the condition $t<I_m(p,q)$. Since $R_n$ is independent of $\sigma$. We only need to focus on $l(\sigma)$.

Note that
\begin{eqnarray*}
l(\sigma)&=&\left[\log \left(\frac{p}{q}\right)\mathbbm{1}{[A_{i_1\dots i_m}=1]}+\log \left(\frac{1-p}{1-q}\right)\mathbbm{1}{[A_{i_1\dots i_m}=0]}\right]\mathbbm{1}{[\sigma_{i_1}=\dots=\sigma_{i_m}=+1]} \\
&&+\left[\log \left(\frac{p}{q}\right)\mathbbm{1}{[A_{i_1\dots i_m}=1]}+\log \left(\frac{1-p}{1-q}\right)\mathbbm{1}{[A_{i_1\dots i_m}=0]}\right]\mathbbm{1}{[\sigma_{i_1}=\dots=\sigma_{i_m}=-1]}.
\end{eqnarray*}
Suppose $i_0\in I_+(\sigma)$ has exactly $m_{1}$ present hyperedges and $m_{2}$ absent hyperedges in $I_{+}(\sigma)$ and $I_{-}(\sigma)$ respectively. Suppose $j_0\in I_-(\sigma)$ has exactly $m_{1}$ present hyperedges and $m_{2}$ absent hyperedges in $I_{+}(\sigma)$ and $I_{-}(\sigma)$ respectively.
Then $l(\sigma)$ remains the same if we flip the label of $i_0$ and $j_0$. Let $\tilde{\sigma}$ be labels obtained from $\sigma$ by flipping the labels of $i_0,j_0$. We shall verify that $l(\sigma)=l(\Tilde{\sigma})$. To prove this, let $T_{1}=\log \left(\frac{p}{q}\right),\;T_{2}=\log \left(\frac{1-p}{1-q}\right),\;$ then 
\begin{eqnarray*}
l(\sigma)
&=&\left(T_{1}\sum\limits_{i_1i_2\dots i_m}\mathbbm{1}{[A_{i_1i_2\dots i_m}=1]}+T_{2}\sum\limits_{i_1i_2\dots i_m}\mathbbm{1}{[A_{i_1i_2\dots i_m}=0]}\right)\mathbbm{1}{[\sigma_{i_1}=\dots=\sigma_{i_m}=+1]} \\
&+&\left(T_{1}\sum\limits_{i_1i_2\dots i_m}\mathbbm{1}{[A_{i_1i_2\dots i_m}=1]}+T_{2}\sum\limits_{i_1i_2\dots i_m}\mathbbm{1}{[A_{i_1i_2\dots i_m}=0]}\right)\mathbbm{1}{[\sigma_{i_1}=\dots=\sigma_{i_m}=-1]}.
\end{eqnarray*}
Further, $l(\sigma)$ can be written as
\begin{eqnarray*}
l(\sigma)&=&T_{1}\sum\limits_{\substack{i_1i_2\dots i_m\in I_+(\sigma) \\ i_1i_2\dots i_m \neq i_0}}\mathbbm{1}{[A_{i_1i_2\dots i_m}=1]}+T_{1}\sum\limits_{\substack{i_2\dots i_m\in I_+(\sigma) \\ i_2\dots i_m \neq i_0}}\mathbbm{1}{[A_{i_0i_2\dots i_m}=1]} \\
&+&T_{2}\sum\limits_{\substack{i_1i_2\dots i_m\in I_+(\sigma) \\ i_1i_2\dots i_m \neq i_0}}\mathbbm{1}{[A_{i_1i_2\dots i_m}=0]}+T_{2}\sum\limits_{\substack{i_2\dots i_m\in I_+(\sigma) \\ i_2\dots i_m \neq i_0}}\mathbbm{1}{[A_{i_0i_2\dots i_m}=0]} \\
&+&T_{1}\sum\limits_{\substack{i_1i_2\dots i_m\in I_-(\sigma) \\ i_1i_2\dots i_m \neq j_0}}\mathbbm{1}{[A_{i_1i_2\dots i_m}=1]}+T_{1}\sum\limits_{\substack{i_2\dots i_m\in I_-(\sigma) \\ i_2\dots i_m \neq j_0}}\mathbbm{1}{[A_{j_0i_2\dots i_m}=1]} \\
&+&T_{2}\sum\limits_{\substack{i_1i_2\dots i_m\in I_-(\sigma) \\ i_1i_2\dots i_m \neq j_0}}\mathbbm{1}{[A_{i_1i_2\dots i_m}=0]}+T_{2}\sum\limits_{\substack{i_2\dots i_m\in I_-(\sigma) \\ i_2\dots i_m \neq j_0}}\mathbbm{1}{[A_{j_0i_2\dots i_m}=0]},
\end{eqnarray*}
and
\begin{eqnarray*}
l(\Tilde{\sigma})&=&T_{1}\sum\limits_{\substack{i_1i_2\dots i_m\in I_+(\sigma) \\ i_1i_2\dots i_m \neq j_0}}\mathbbm{1}{[A_{i_1i_2\dots i_m}=1]}+T_{1}\sum\limits_{\substack{i_2\dots i_m\in I_+(\sigma) \\ i_2\dots i_m \neq j_0}}\mathbbm{1}{[A_{j_0i_2\dots i_m}=1]} \\
&+&T_{2}\sum\limits_{\substack{i_1i_2\dots i_m\in I_+(\sigma) \\ i_1i_2\dots i_m \neq j_0}}\mathbbm{1}{[A_{i_1i_2\dots i_m}=0]}+T_{2}\sum\limits_{\substack{i_2\dots i_m\in I_+(\sigma) \\ i_2\dots i_m \neq j_0}}\mathbbm{1}{[A_{j_0i_2\dots i_m}=0]} \\
&+&T_{1}\sum\limits_{\substack{i_1i_2\dots i_m\in I_-(\sigma) \\ i_1i_2\dots i_m \neq i_0}}\mathbbm{1}{[A_{i_1i_2\dots i_m}=1]}+T_{1}\sum\limits_{\substack{i_2\dots i_m\in I_-(\sigma) \\ i_2\dots i_m \neq i_0}}\mathbbm{1}{[A_{i_0i_2\dots i_m}=1]} \\
&+&T_{2}\sum\limits_{\substack{i_1i_2\dots i_m\in I_-(\sigma) \\ i_1i_2\dots i_m \neq i_0}}\mathbbm{1}{[A_{i_1i_2\dots i_m}=0]}+T_{2}\sum\limits_{\substack{i_2\dots i_m\in I_-(\sigma) \\ i_2\dots i_m \neq i_0}}\mathbbm{1}{[A_{i_0i_2\dots i_m}=0]}
\end{eqnarray*}
Then $l(\sigma)=l(\Tilde{\sigma})$ by the assumption of $i_0$ and $j_0$.

Next we will show there are $k_n$ ($k_n\rightarrow\infty$) such pairs. More specifically, we will show that there exists $i_{1},i_{2},\dots,i_{k}\in I_{+}(\sigma)$ and $j_{1},j_{2},\dots,j_{k}\in I_{-}(\sigma)$ with $k\gg 1$ such that the likelihood function remains unchanged if we flip the label of a pair $(i_{t},j_{t})$ $t=1,2,\dots,k$. Let $\eta_t$ be the label obtained by flipping the label of $i_{t},j_{t}$ in $\sigma$. Then $l(\eta_t)=l(\sigma)$ for $1\leq t\leq k\rightarrow\infty$.

Let $n_{1}=|I_{+}(\sigma)|$ and $n_{2}=|I_{-}(\sigma)|$. Then $n_{1},n_{2}=\frac{n}{2}(1+O(n^{-\frac{1}{3}}))$ with probability $1-o(1)$. Hence we can take $n_{1}=n_{2}=\frac{n}{2}$ below.
Let $S_{+}\subset I_{+}(\sigma)$ be a random subset with $|S_{+}|=\frac{n}{\log^{2}n}$ and 
$S_{-}\subset I_{-}(\sigma)$ be a random subset with $|S_{-}|=\frac{n}{\log^{2}n}$. Denote $S=S_{+}\cup S_{-}.$ Define 
\[S_{0}=\left\{i\in S|any\;i_{2},\dots,i_{t}\in S,i_{t+1},\dots,i_{m}\in S^{c},\;s.t.\;A_{ii_{2}\dots i_{t}i_{t+1}\dots i_{m}}= *,\;t\geq 2\right\}.
\] 
 For each node $i\in S_0$, hyperedge $A_{ii_2\dots i_m}$ is possibly revealed if and only if $\{i_2,\dots ,i_m\}\subset I_{+}(\sigma)-S$ or $\{i_2,\dots ,i_m\}\subset I_{-}(\sigma)-S$.
 
We will show $|S_{0}|=\frac{2n(1+o(1))}{\log^{2}n}$ with probability $1-o(1).$ Let 
\begin{center}
$T=\sum\limits_{t=2}^{m}\sum\limits_{\substack{i_{1},\dots,i_{t}\in S \\ i_{t+1},\dots,i_{m}\in S^{c}}}\mathbbm{1}{[A_{i_{1}i_{2}\dots i_{t}i_{t+1}\dots i_{m}}\neq *]}.$
\end{center}
The expectation of $T$ is
\begin{eqnarray*}
\mathbb{E}T&=&\sum\limits_{t=2}^{m}\binom{\frac{2n}{\log^{2}n}}{t}\binom{n-\frac{2n}{\log^{2}n}}{m-t}\alpha \\
&=&\sum\limits_{t=2}^{m}\binom{\frac{2n}{\log^{2}n}}{t}\binom{n-\frac{2n}{\log^{2}n}}{m-t}\frac{t\log n}{n^{m-1}} \\
&=&\frac{c\cdot n^{m}}{\log^{4}n}\frac{t\log n}{n^{m-1}} \\
&\asymp&\frac{n}{\log^{3}n}.
\end{eqnarray*}
Hence, by Markov inequality we have
\[\mathbb{P}\left(T\geq \frac{n}{\log^{2}n\sqrt{\log n}}\right)\leq \frac{1}{\frac{n}{\log^{2}n\sqrt{\log n}}}\frac{c\cdot n}{\log^{3}n}=\frac{\sqrt{\log n}}{\log n}=o(1).\]
Then $T<\frac{n}{\log^{2}n\sqrt{\log n}}$ with probablity $1-o(1).$ 
Hence $|S_{0}|=\frac{2n}{\log^{2}n}(1+o(1))$ with probability $1-o(1)$.
\\

Let $m_{1}=\frac{\sqrt{pq}t\log n}{2^{m-1}(m-1)!}$, $m_{2}=\frac{\sqrt{(1-p)(1-q)}t\log n}{2^{m-1}(m-1)!}.$ For some $k\gg 1,$ we will show that there exists $i_{t}\in S_{0}\cap S_{+},\; (1\leq t\leq k)$ such that $i_{t}$ has $m_{1}$ present hyperedges and $m_{2}$ absent hyperedges in $I_{+}(\sigma)$ and $I_{-}(\sigma)$ respectively. Denote
\begin{center}
$\widetilde{n}_{1}=\binom{n_{1}-\frac{2n}{\log^{2}n}}{m-1}\sim \frac{n^{m-1}}{2^{m-1}(m-1)!}.$
\end{center}
Let $i_{0}\in S_{0}\cap S_{+},$ the probability that $i_{0}$ has $m_{1}$ present hyperedges, $m_{2}$ absent hyperedges in $I_{+}(\sigma)$ and $I_{-}(\sigma)$ respectively is,
\begin{eqnarray*}
p_{0}&=&\frac{\widetilde{n}_{1}!}{m_{1}!m_{2}!(\widetilde{n}_1-m_{1}-m_{2})!}\cdot (\alpha p)^{m_{1}}[\alpha (1-p)]^{m_{2}}(1-\alpha)^{(\widetilde{n}_1-m_{1}-m_{2})} \\
&&\times\frac{\widetilde{n}_1!}{m_{1}!m_{2}!(\widetilde{n}_1-m_{1}-m_{2})!}\cdot (\alpha q)^{m_{1}}[\alpha (1-q)]^{m_{2}}(1-\alpha)^{(\widetilde{n}_1-m_{1}-m_{2})} \\
&\sim&\frac{1}{m_{1}!^{2}m_{2}!^{2}}\left[\frac{\widetilde{n}_1^{\widetilde{n}_1+\frac{1}{2}}e^{-\widetilde{n}_1}}{(\widetilde{n}_1-m_{1}-m_{2})^{\widetilde{n}_1-m_{1}-m_{2}+\frac{1}{2}}e^{-\widetilde{n}_1+m_{1}+m_{2}}}\right]^{2}(\alpha^{2}pq)^{m_{1}}[\alpha^{2}(1-p)(1-q)]^{m_{2}} (1-\alpha)^{2(\widetilde{n}_1-m_{1}-m_{2})}\\
&=&\frac{1}{m_{1}!^{2}m_{2}!^{2}}\left[\frac{(\widetilde{n}_1-m_{1}-m_{2})^{m_{1}+m_{2}}}{e^{m_{1}+m_{2}}(1-\frac{m_{1}+m_{2}}{\widetilde{n}_1})^{\widetilde{n}_1+\frac{1}{2}}}\right]^{2}(\alpha^{2}pq)^{m_{1}}[\alpha^{2}(1-p)(1-q)]^{m_{2}}(1-\alpha)^{2(\widetilde{n}_1-m_{1}-m_{2})} \\
&=&\frac{1}{m_{1}!^{2}m_{2}!^{2}}\left[\frac{\widetilde{n}_1^{m_{1}+m_{2}}}{e^{m_{1}+m_{2}}e^{-(m_{1}+m_{2})}}\right]^{2}(\alpha^{2}pq)^{m_{1}}\left[\alpha^{2}(1-p)(1-q)\right]^{m_{2}}e^{-\frac{t\log n}{2^{m-2}(m-1)!}} \\
&=&\frac{\tilde{n}_{1}^{2(m_{1}+m_{2})}}{m_{1}!^{2}m_{2}!^{2}}e^{-\frac{t\log n}{2^{m-2}(m-1)!}}(\alpha^{2}pq)^{m_{1}}[\alpha^{2}(1-p)(1-q)]^{m_{2}} \\
&=&\frac{n^{-\frac{t}{2^{m-2}(m-1)!}}}{m_{1}!^{2}m_{2}!^{2}}(\alpha^{2}\tilde{n}_{1}^2pq)^{m_{1}}[\alpha^{2}\tilde{n}_{1}^2(1-p)(1-q)]^{m_{2}} \\
&=&n^{-\frac{t}{2^{m-2}(m-1)!}}\frac{e^{2(m_{1}+m_{2})}}{4\pi^{2}m_{1}m_{2}}(\frac{\alpha^{2}\tilde{n}_{1}^2pq}{m_{1}^{2}})^{m_{1}}(\frac{\alpha^{2}\tilde{n}_{1}^2(1-p)(1-q)}{m_{2}^{2}})^{m_{2}} \\
&=&\frac{1}{4\pi^{2}m_{1}m_{2}}n^{-\frac{t}{2^{m-2}(m-1)!}}e^{\frac{\sqrt{pq}+\sqrt{(1-p)(1-q)}}{2^{m-2}(m-1)!}t\log n} \\
&=&\frac{1}{4\pi^{2}m_{1}m_{2}}n^{-\frac{t}{2^{m-2}(m-1)!}[1-\sqrt{pq}-\sqrt{(1-p)(1-q)}]} \\
&=&\frac{1}{4\pi^{2}m_{1}m_{2}}n^{-t\cdot \frac{(\sqrt{p}-\sqrt{q})^{2}+(\sqrt{1-p}-\sqrt{1-q})^{2}}{2^{m-1}(m-1)!}}.
\end{eqnarray*}
If $t<\frac{2^{m-1}(m-1)!}{(\sqrt{p}-\sqrt{q})^{2}+(\sqrt{1-p}-\sqrt{1-q})^{2}},$ then $p_{0}\gg \frac{n^{1-\epsilon}}{n}$ for some $\epsilon\in (0,1).$ 
Similarly, the probability that $j_{0}\in S_{0}\cap S_{-}$ has $m_{1}$ present hyperedges and $m_{2}$ absent hyperedges in $I_{+}(\sigma)$ and $I_{-}(\sigma)$ is equal to $p_{0}.$

For $i\in S_{0},$ let $\mathbbm{1}_{i}$ denote the event that $i$ has $m_{1}$ present hyperedges and $m_{2}$ absent hyperedges in $I_{+}(\sigma)$ and $I_{-}(\sigma)$ respectively. Define two random variables
\[X=\sum\limits_{i\in S_{0}\cap S_{+}}\mathbbm{1}_{i},\ \ \ \ \ \ Y=\sum\limits_{i\in S_{0}\cap S_{-}}\mathbbm{1}_{i}.\]
If
$\mathbbm{1}_{i}=\mathbbm{1}_{j}=1$   for $i\in S_{0}\cap S_{+}$ and $j\in S_{0}\cap S_{-},$
then the likelihood function remains unchanged if we flip the labels of $i$ and $j$.
By Chebyshev's inequality, given $|S_{0}\cap S_{+}|$, we have
\begin{eqnarray*}
&&\mathbb{P}\left(X\leq (1-\epsilon)\frac{2n}{\log^{2}n}p_{0}\right)\\
&=&\mathbb{P}\left(X\leq (1-\epsilon)\frac{2n}{\log^{2}n}p_{0}\middle| |S_{0}\cap S_{+}|\geq \frac{2n}{\log^{2}n}(1-o(1))\right)\cdot \mathbb{P}\left(|S_{0}\cap S_{+}|\geq \frac{2n}{\log^{2}n}(1-o(1))\right) \\
&+&\mathbb{P}\left(X\leq (1-\epsilon)\frac{2n}{\log^{2}n}p_{0}\middle| |S_{0}\cap S_{+}|<\frac{2n}{\log^{2}n}(1-o(1))\right)\mathbb{P}\left(|S_{0}\cap S_{+}|<\frac{2n}{\log^{2}n}\right) \\
&\leq&\mathbb{P}\left(X\leq (1-\epsilon)|S_{0}\cap S_{+}|p_{0}\middle| |S_{0}\cap S_{+}|\geq \frac{2n}{\log^{2}n}(1-o(1))\right) + o(1) \\
&\leq&\frac{1}{\epsilon^{2}|S_{0}\cap S_{+}|p_{0}}+o(1).
\end{eqnarray*}
Since $p_{0}\gg \frac{n^{1-\epsilon}}{n}$ for some $\epsilon>0$ and $|S_{0}\cap S_{+}|\geq \frac{2n}{\log^{2}n}(1-o(1)).$ Then $X\geq |S_{0}\cap S_{+}|p_{0} \rightarrow +\infty$ with probability $1-o(1).$ 
Similarly $Y\geq |S_{0}\cap S_{+}|p_{0} \rightarrow +\infty$ with probability $1-o(1).$
As a result, we have pairs $(i_{t},j_{t})$ $(1\leq t\leq k\rightarrow\infty)$. For each $t$, the likelihood is constant by flipping the labels of $i_{t}$ and $j_{t}$.
The proof is complete by Proposition 4.1 in \cite{DGMS21}.

\section{Proof of Theorem \ref{MLEpossible} }\label{Likelihood}
\textit{Proof of Theorem \ref{MLEpossible} :}
Let $\sigma$ be the maximum likelihood estimator(MLE). Recall the log-likelihood function in (\ref{loglike}). The MLE fails to exactly recover the true label if there exists a label $\eta$ such that $l(\eta)\geq l(\sigma)$ with probability $\delta$ for some constant $\delta>0$. Our proof proceeds by showing that the probability MLE fails is $o(1)$.

The maximum likelihood estimator(MLE) is obtained by maximizing $\log L$ in (\ref{loglike}) with respect to $\sigma$. The first term of $\log L$ does not involve $\sigma$.  Hence we only need to maximize the second term of $\log L$ to get MLE. 
Let $\sigma$ be the MLE.
Recall that the MLE fails if there exists a label $\eta$ such that $l(\eta)\geq l(\sigma)$ with probability $\delta$ for some constant $\delta>0$. Below, we show the probability MLE fails is $o(1)$.

Let $k$ be an even number and $1\leq k\leq \frac{n}{2}$. Define the Hamming distance between two labels $\sigma,\eta$ as 
\[d(\sigma,\eta)=\min\left\{\sum\limits_{i=1}^{n}\mathbbm{1}{[\sigma_{i}\neq \eta_{i}]}, \sum\limits_{i=1}^{n}\mathbbm{1}{[\sigma_{i}\neq -\eta_{i}]}\right\}.\]
Let $\eta$ be a label such that $d(\sigma,\eta)=k,$ and denote
\begin{center}
$C_{i_1i_2\dots i_m}(A)=\log \left(\frac{p}{q}\right)\mathbbm{1}{[A_{i_1i_2\dots i_m}=1]}+\log \left(\frac{1-p}{1-q}\right)\mathbbm{1}{[A_{i_1i_2\dots i_m}=0]}.$
\end{center}
Then log-likelihood difference at $\eta$ and $\sigma$ is
\[l(\eta)-l(\sigma)=\sum\limits_{1\leq i_1<\dots<i_m\leq n}C_{i_1i_2\dots i_m}(A)(\mathbbm{1}_{i_1\dots i_m}(\eta)-\mathbbm{1}_{i_1\dots i_m}(\sigma)).\]
We will show
\[\mathbb{P}(\exists k\ and\ d(\sigma,\eta)=k\ ,s.t.\; l(\eta)-l(\sigma)\geq 0)=o(1).\]
Recall $\textrm{I}_+(\sigma)$ and $\textrm{I}_-(\sigma)$. Denote $\mathbbm{1}_{i_1\dots i_m}(\eta)=I[\eta_{i_1}=\eta_{i_2}=\dots=\eta_{i_m}]$. Note that
\begin{equation*}
\mathbbm{1}_{i_1\dots i_m}(\eta)-\mathbbm{1}_{i_1\dots i_m}(\sigma)=\left\{
\begin{array}{cc}
    1 ,& \ \ \   {i_1\dots i_m}\subset \textrm{I}_+(\eta)\ or\ \textrm{I}_-(\eta),\;{i_1\dots i_m}\not\subset \textrm{I}_+(\sigma)\ ,\textrm{I}_-(\sigma); \\
    -1 ,&\ \ \  {i_1\dots i_m}\subset \textrm{I}_+(\sigma)\ or\ \textrm{I}_-(\sigma),\;{i_1\dots i_m}\not\subset \textrm{I}_+(\eta),\textrm{I}_-(\eta); \\
    0, & otherwise.
\end{array}
\right.
\end{equation*}
Hence, $l(\eta)-l(\sigma)$ is written as
\begin{center}
$l(\eta)-l(\sigma)=\sum\limits_{\substack{i_1\dots i_m \\ {i_1\dots i_m}\subset I_{+}(\eta)\ or\ \textrm{I}_-(\eta) \\ {i_1\dots i_m}\not\subset I_{+}(\sigma), I_{-}(\sigma)}}C_{i_1\dots i_m}(A)-\sum\limits_{\substack{i_1\dots i_m \\ {i_1\dots i_m}\subset I_{+}(\sigma)\ or\ \textrm{I}_-(\sigma) \\ {i_1\dots i_m}\not\subset I_{+}(\eta), I_{-}(\eta)}}C_{i_1\dots i_m}(A).$
\end{center}
It is easy to verify that there are $n_{k}=2\left[\binom{\frac{n}{2}}{m}-\binom{\frac{k}{2}}{m}-\binom{\frac{n-k}{2}}{m}\right]$ hyperedges $\{i_{1},\dots ,i_{m}\}$ such that $\{i_{1}\dots i_{m}\}\subset \mathbbm{1}_{+}(\eta)\ or\ \mathbbm{1}_{-}(\eta), \{i_{1}\dots i_{m}\}\not\subset \mathbbm{1}_{+}(\sigma),\mathbbm{1}_{-}(\sigma)$. For convenience,
define random variables $X, Y$ as
\[\mathbb{P}(X=1)=\alpha p,\quad \mathbb{P}(X=0)=\alpha(1-p),\quad \mathbb{P}(X=-1)=1-\alpha.\]
\[
\mathbb{P}(Y=1)=\alpha q,\quad \mathbb{P}(Y=0)=\alpha(1-q),\quad \mathbb{P}(Y=-1)=1-\alpha.\]
Let $X_{i}, Y_{i}$ be $i.i.d$ copies of $X, Y$ respectively and
\begin{center}
$W_{i}=\log \left(\frac{p}{q}\right)\mathbbm{1}{[X_{i}=1]}+\log \left(\frac{1-p}{1-q}\right)\mathbbm{1}{[X_{i}=0]}$ \\
$V_{i}=\log \left(\frac{p}{q}\right)\mathbbm{1}{[Y_{i}=1]}+\log \left(\frac{1-p}{1-q}\right)\mathbbm{1}{[Y_{i}=0]}.$
\end{center}
For any $r>0$, by Markov inequality we have
\begin{eqnarray*}
\mathbb{P}(l(\eta)-l(\sigma)\geq 0)&=&\mathbb{P}\left(\sum\limits_{i=1}^{n_{k}}(V_{i}-W_{i})\geq 0\right) \\
&=&\mathbb{P}\left(\sum\limits_{i=1}^{n_{k}}(W_{i}-V_{i})\leq 0\right) \\
&=&\mathbb{P}\left(e^{\sum\limits_{i=1}^{n_{k}}(-r)(W_{i}-V_{i})}\geq 1\right) \\
&\leq&\left[\mathbb{E}\left(e^{-rW_{1}}\right)\mathbb{E}\left(e^{rV_{1}}\right)\right]^{n_{k}}.
\end{eqnarray*}
Next, we find the explicit expression of expectations $\mathbb{E}\left(e^{-rW_{1}}\right)$ and $\mathbb{E}\left(e^{rV_{1}}\right)$.
\begin{eqnarray*}
\mathbb{E}[e^{-rW_{1}}]&=&\mathbb{E}e^{-r\left(\log \left(\frac{p}{q}\right)\mathbbm{1}{[X_{i}=1]}+\log \left(\frac{1-p}{1-q}\right)\mathbbm{1}{[X_{i}=0]}\right)} \\
&=&e^{-r\log (\frac{p}{q})}\alpha p+e^{-r\log (\frac{1-p}{1-q})}\alpha (1-p)+(1-\alpha) \\
&=&\left(\frac{q}{p}\right)^{r}\alpha p+\left(\frac{1-q}{1-p}\right)^{r}\alpha (1-p)+(1-\alpha)
\end{eqnarray*}
\begin{eqnarray*}
\mathbb{E}[e^{rV_{1}}]&=&\mathbb{E}e^{r\left(\log (\frac{p}{q})\mathbbm{1}{[Y_{i}=1]}+\log (\frac{1-p}{1-q})\mathbbm{1}{[Y_{i}=0]}\right)} \\
&=&e^{r\log (\frac{p}{q})}\alpha q+e^{r\log (\frac{1-p}{1-q})}\alpha (1-q)+(1-\alpha) \\
&=&\left(\frac{p}{q}\right)^{r}\alpha q+\left(\frac{1-p}{1-q}\right)^{r}\alpha (1-q)+(1-\alpha)
\end{eqnarray*}
Taking $r=\frac{1}{2}$ yields
\begin{eqnarray*}
\mathbb{E}[e^{-rW_{1}}]&=&\alpha\sqrt{pq}+\alpha\sqrt{(1-p)(1-q)}+(1-\alpha) \\
&=&1+\alpha[\sqrt{pq}+\sqrt{(1-p)(1-q)}-1],
\\
\mathbb{E}[e^{rV_{1}}]&=&\alpha\sqrt{pq}+\alpha\sqrt{(1-p)(1-q)}+(1-\alpha) \\
&=&1+\alpha[\sqrt{pq}+\sqrt{(1-p)(1-q)}-1].
\end{eqnarray*}
Hence,
\begin{eqnarray}\nonumber
\log \mathbb{P}\left(l(\eta)-l(\sigma)\geq 0\right)&\leq&n_{k}\log \mathbb{E}[ e^{-rW_{1}}] +n_{k}\log \mathbb{E}[e^{rV_{1}}] \\ \nonumber
&\leq&n_{k}[2\alpha(\sqrt{pq}+\sqrt{(1-p)(1-q)}-1)] \\ \nonumber
&=&n_{k}\alpha \left[(-1)\left\{(\sqrt{p}-\sqrt{q})^{2}+(\sqrt{1-p}-\sqrt{1-q})^{2}\right\}\right] \\  \label{logpd}
&=&-n_{k}\alpha \left[(\sqrt{p}-\sqrt{q})^{2}+(\sqrt{1-p}-\sqrt{1-q})^{2}\right].
\end{eqnarray}
For $k\geq\frac{n}{\log \log n}$, it is easy to check $n_{k}\geq \frac{1}{2^{m-1}}\frac{n}{\log \log n}\binom{n-1}{m-1}$. Hence by (\ref{logpd}), we get
\begin{eqnarray*}
\mathbb{P}\left(l(\eta)-l(\sigma)\geq 0\right)&\leq&e^{-[(\sqrt{p}-\sqrt{q})^{2}+(\sqrt{1-p}-\sqrt{1-q})^{2}]\frac{t\log n}{n^{m-1}}\frac{1}{2^{m-1}}\frac{n}{\log \log n}\frac{n^{m-1}}{(m-1)!}} \\
&=&e^{-[(\sqrt{p}-\sqrt{q})^{2}+(\sqrt{1-p}-\sqrt{1-q})^{2}]\frac{t}{2^{m-1}(m-1)!}\frac{n\log n}{\log \log n}} \\
&=&e^{-c\frac{n\log n}{\log \log n}},
\end{eqnarray*}
for some positive constant $c$. Clearly, there are $\binom{\frac{n}{2}}{\frac{k}{2}}^{2}$ many choices for $\eta$ with $d(\sigma,\eta)=k$. Note that $ \binom{\frac{n}{2}}{\frac{k}{2}}^{2}\leq 2^{n}$. 
Then the probability that there exists $\eta$ with $d(\sigma,\eta)=k$ for $k\geq\frac{n}{\log \log n}$ is upper bounded by
\[\frac{n}{2}\cdot 2^{n}\cdot e^{-c\frac{n\log n}{\log \log n}}=e^{n\log 2 + \log\frac{n}{2} - cn\frac{\log n}{\log \log n}}=o(1).\]

For $k<\frac{n}{\log \log n}$, we have $n_{k}=\frac{k}{2^{m-1}}\binom{n-1}{m-1}.$ Then
\begin{eqnarray*}
\mathbb{P}\left(l(\eta)-l(\sigma)\geq 0\right)&\leq&e^{-[(\sqrt{p}-\sqrt{q})^{2}+(\sqrt{1-p}-\sqrt{1-q})^{2}]\frac{t\log n}{n^{m-1}}\frac{k}{2^{m-1}}\frac{n^{m-1}}{(m-1)!}} \\
&=&e^{-\frac{(\sqrt{p}-\sqrt{q})^{2}+(\sqrt{1-p}-\sqrt{1-q})^{2}}{2^{m-1}(m-1)!}tk\log n} \\
&=&n^{-\frac{[\sqrt{p}-\sqrt{q}]^{2}+[\sqrt{1-p}-\sqrt{1-q}]^{2}}{2^{m-1}(m-1)!}tk}.
\end{eqnarray*}
There are $\binom{\frac{n}{2}}{\frac{k}{2}}^{2}\leq n^{k}$ many choices for $\eta$ with $d(\sigma,\eta)=k$.  Then the probability that there exists $\eta$ with $d(\sigma,\eta)=k$ for $k<\frac{n}{\log \log n}$ is upper bounded by
\begin{eqnarray*}
k\cdot \binom{\frac{n}{2}}{\frac{k}{2}}^{2}\mathbb{P}\left(l(\eta)-l(\sigma)\geq 0\right)&\leq&kn^{k}\cdot n^{-\frac{[\sqrt{p}-\sqrt{q}]^{2}+[\sqrt{1-p}-\sqrt{1-q}]^{2}}{2^{m-1}(m-1)!}tk} \\
&\leq&kn^{k}n^{-(1+\epsilon)k} \\
&=&\frac{k}{n^{\epsilon k}}=o(1),
\end{eqnarray*}
where $\epsilon$ is a constant such that $\frac{[\sqrt{p}-\sqrt{q}]^{2}+[\sqrt{1-p}-\sqrt{1-q}]^{2}}{2^{m-1}(m-1)!}t=1+\epsilon$. This is possible by the condition
$t>\frac{2^{m-1}(m-1)!}{(\sqrt{p}-\sqrt{q})^{2}+(\sqrt{1-p}-\sqrt{1-q})^{2}}$. Then the proof is complete.

\section{Proof of Theorem \ref{polytime}}\label{Six}
The proof proceeds by showing the probability that there exists a mislabelled node goes to zero.
By the definition of
hypergraph $\Tilde{A}$, we have
\begin{eqnarray*}
\mathbb{P}(\Tilde{A}_{i_1i_2\dots i_m}=1)
&=&\left\{
\begin{array}{cc}
    \frac{\log \log n}{\log n}\cdot \alpha p, & \{i_1,i_2,\dots,i_m\}\subset I_+(\sigma)\ or \  I_-(\sigma),\\
    \frac{\log \log n}{\log n}\cdot \alpha q, & \textrm{otherwise}.
\end{array}
\right.\\
&=&\left\{
\begin{array}{cc}
    \frac{tp\log \log n}{n^{m-1}}, & \{i_1,i_2,\dots,i_m\}\subset I_+(\sigma)\ or \  I_-(\sigma),\\
    \frac{tq\log \log n}{n^{m-1}}, & \textrm{otherwise}.
\end{array}
\right.
\end{eqnarray*}
Then $\tilde{A}$ has the same community structure as the original hypergraph $A$ and in $\tilde{A}$, the order of hyperedge probability is $\frac{\log\log n}{n^{m-1}}$. With probability $1-o(1)$, the weak recovery algorithm in \cite{ALS18} will recover the true labels of $(1-\delta) n$ nodes of $\tilde{A}$ with $\delta=o(1)$. Denote the communities as $\Tilde{I}_+(\sigma)$, $\Tilde{I}_-(\sigma)$.  
Hence, with probability $1-o(1)$, there are $\frac{\delta}{2}n$ nodes in $\Tilde{I}_+(\sigma)$ and $\Tilde{I}_-(\sigma)$ that are mislabelled. By the refinement step, a node $i$ among the correctly labelled $\frac{1-\delta}{2}n$ nodes in $ \Tilde{I}_+(\sigma)$ is mislabelled if 
\[e(i, \Tilde{I}_+(\sigma))<e(i, \Tilde{I}_-(\sigma)).\]
A node among the mislabelled $\frac{\delta}{2}n$ nodes in $\Tilde{I}_+(\sigma)$ remains mislabelled if 
\[e(i, \Tilde{I}_+(\sigma))\geq e(i, \Tilde{I}_-(\sigma)).\]
 Similar result holds for nodes in $\Tilde{I}_-(\sigma)$.
Let $X_i, Y_i, W_i, V_i$ be defined as in the proof of Theorem \ref{MLEpossible} and $W_i', V_i'$ be $i.i.d.$ copies of $W_i, V_i$. Then a node $i$ is mislabelled is equivalent to 
\[\substack{\binom{\frac{\delta}{2}n}{m-1} \\ \sum\limits_{i=1}W_i}+\substack{\binom{\frac{n}{2}}{m-1}-\binom{\frac{\delta}{2}n}{m-1} \\ \sum\limits_{i=1}V_i}\geq \substack{\binom{(1-\delta)\frac{n}{2}}{m-1} \\ \sum\limits_{i=1}W_i'}+\substack{\binom{\frac{n}{2}}{m-1}-\binom{(1-\delta)\frac{n}{2}}{m-1} \\ \sum\limits_{i=1}V_i'}.\]

We are going to bound the probability that node $i$ is mislabelled and then apply the union bound. Let $r=\frac{1}{\delta \sqrt{\log (\frac{1}{\delta})}}$. Then we have
\begin{eqnarray*}
p_i&=&\mathbb{P}(\textrm{node $i$ is mislabelled}) \\
&=&\mathbb{P}\left[\substack{\binom{\frac{\delta}{2}n}{m-1} \\ \sum\limits_{i=1}W_i}+\substack{\binom{\frac{n}{2}}{m-1}-\binom{\frac{\delta}{2}n}{m-1} \\ \sum\limits_{i=1}V_i}\geq \substack{\binom{(1-\delta)\frac{n}{2}}{m-1} \\ \sum\limits_{i=1}W_i'}+\substack{\binom{\frac{n}{2}}{m-1}-\binom{(1-\delta)\frac{n}{2}}{m-1} \\ \sum\limits_{i=1}V_i'}\right] \\
&=&\mathbb{P}\left[\substack{\binom{\frac{n}{2}}{m-1}-\binom{\frac{\delta}{2}n}{m-1} \\ \sum\limits_{i=1}(V_i-W_i')}+\substack{\binom{\frac{\delta}{2}n}{m-1} \\ \sum\limits_{i=1}W_i}\geq \substack{\binom{\frac{n}{2}}{m-1}-\binom{(1-\delta)\frac{n}{2}}{m-1} \\ \sum\limits_{i=1}V_i'}-\substack{\binom{\frac{n}{2}}{m-1}-\binom{\frac{\delta n}{2}}{m-1}-\binom{(1-\delta)\frac{n}{2}}{m-1} \\ \sum\limits_{i=1}W_i'}\right] \\
&\leq&\mathbb{P}\left[\substack{\binom{\frac{n}{2}}{m-1}-\binom{\frac{\delta}{2}n}{m-1} \\ \sum\limits_{i=1}(V_i-W_i')}\geq -r\delta \log n\right] + \mathbb{P}\left[\substack{\binom{\frac{\delta}{2}n}{m-1} \\ \sum\limits_{i=1}W_i}+\substack{\binom{\frac{n}{2}}{m-1}-\binom{\frac{\delta n}{2}}{m-1}-\binom{(1-\delta)\frac{n}{2}}{m-1} \\ \sum\limits_{i=1}W_i'}-\substack{\binom{\frac{n}{2}}{m-1}-\binom{(1-\delta)\frac{n}{2}}{m-1} \\ \sum\limits_{i=1}V_i'}\geq r\delta \log n\right] \\
&=&(I)+(II).
\end{eqnarray*}
Next we show $(II)=O\left(n^{-2}\right)$ and $(I)=O\left(n^{-\frac{t}{I_m(p,q)}}\right)$. It is easy to verify that
\[(II)\leq \mathbb{P}\left(\substack{\binom{\frac{\delta}{2}n}{m-1} \\ \sum\limits_{i=1}W_i}\geq \frac{r\delta}{3}\log n\right)+\mathbb{P}\left(\substack{\binom{\frac{n}{2}}{m-1}-\binom{\frac{\delta n}{2}}{m-1}-\binom{(1-\delta)\frac{n}{2}}{m-1} \\ \sum\limits_{i=1}W_i'}\geq \frac{r\delta}{3}\log n\right)+\mathbb{P}\left(\substack{\binom{\frac{n}{2}}{m-1}-\binom{(1-\delta)\frac{n}{2}}{m-1} \\ \sum\limits_{i=1}-V_i'}\geq \frac{r\delta}{3}\log n\right).\]
Since $p>q>0$, it follows that $1-q>1-p$ and then
\begin{eqnarray*}
W_i&=&\log \left(\frac{p}{q}\right)\mathbbm{1}{[X_i=1]}+\log \left(\frac{1-p}{1-q}\right)\mathbbm{1}{[X_i=0]} \\
&\leq&\log \left(\frac{p}{q}\right)\mathbbm{1}{[X_i=1]}.
\end{eqnarray*}
Then by the multiplicative Chernoff bound, one has
\begin{eqnarray*}
\mathbb{P}\left(\substack{\binom{\frac{\delta}{2}n}{m-1} \\ \sum\limits_{i=1}W_i}\geq \frac{r\delta}{3}\log n\right)&\leq&\mathbb{P}\left(\substack{\binom{\frac{\delta}{2}n}{m-1} \\ \sum\limits_{i=1}\mathbbm{1}{[X_i=1]}}\geq \frac{r\delta \log n}{3\log (\frac{p}{q})}\right) \\
&\leq&\left(\frac{\frac{r}{\delta^{m-2}}2^{m-1}(m-1)!}{e\cdot 3pt\log (\frac{p}{q}) }\right)^{-\frac{r\delta \log n}{3\log (\frac{p}{q})}} 
\end{eqnarray*}
\begin{eqnarray*}
&=&e^{-\frac{\log n}{3\log (\frac{p}{q}) \sqrt{\log (\frac{1}{\delta})}}[\log (\frac{1}{\delta})+(m-2)\log (\frac{1}{\delta}) (1+o(1))]} \\
&=&e^{-\frac{(m-1)\sqrt{\log (\frac{1}{\delta})}}{3\log (\frac{p}{q})}\log n    \;(1+o(1))} \\
&=&O\left(n^{-2}\right).
\end{eqnarray*}
Similarly, we get 
\[\mathbb{P}\left(\substack{\binom{\frac{n}{2}}{m-1}-\binom{\frac{\delta n}{2}}{m-1}-\binom{(1-\delta)\frac{n}{2}}{m-1} \\ \sum\limits_{i=1}W_i'}\geq \frac{r\delta}{3}\log n\right)=O\left(n^{-2}\right).\]
Note that
\begin{eqnarray*}
-V_i'&=&\log \left(\frac{1-p}{1-q}\right)\mathbbm{1}{[A_{i}=0]}-\log \left(\frac{p}{q}\right)\mathbbm{1}{[A_{i}=1]} \\
&\leq&\log \left(\frac{1-p}{1-q}\right)\mathbbm{1}{[A_{i}=0]}.
\end{eqnarray*}
Hence, by the multiplicative Chernoff bound, it follows that
\begin{eqnarray*}
\mathbb{P}\left(\substack{\binom{\frac{n}{2}}{m-1}-\binom{(1-\delta)\frac{n}{2}}{m-1} \\ \sum\limits_{i=1}(-V_i')}\geq \frac{r\delta}{3}\log n\right)
&\leq&\mathbb{P}\left(\substack{\binom{\frac{n}{2}}{m-1}-\binom{(1-\delta)\frac{n}{2}}{m-1} \\ \sum\limits_{i=1}\mathbbm{1}{[A_{i}=0]}}\geq \frac{r\delta \log n}{3\log (\frac{1-q}{1-p})}\right) \\
&\leq&\left(\frac{\frac{r}{\delta^{m-2}}2^{m-1}(m-1)!}{e\cdot 3(1-p)t\log (\frac{1-q}{1-p}) }\right)^{-\frac{r\delta \log n}{3\log (\frac{1-q}{1-p})}} \\
&=&e^{-\frac{1-\delta\log n}{3\log (\frac{1-q}{1-p})}[(m-1)\log (\frac{1}{\delta}) (1+o(1))]} \\
&=&e^{-\frac{(m-1)\sqrt{\log (\frac{1}{\delta})}\log n}{3\log (\frac{1-q}{1-p})}(1+o(1))} \\
&=&O\left(n^{-2}\right).
\end{eqnarray*}
Then we conclude that $(II)=O\left(n^{-2}\right)$.

Next we bound $(I)$. Note that $\binom{\frac{n}{2}}{m-1}-\binom{\frac{\delta}{2}n}{m-1}=\frac{n^{m-1}}{2^{m-1}(m-1)!}(1+o(1))$. By Markov's inequality, one has
\begin{eqnarray*}
(I)&=&\mathbb{P}\left[e^{\substack{\binom{\frac{n}{2}}{m-1}-\binom{\frac{\delta}{2}n}{m-1} \\ \frac{1}{2}\sum\limits_{i=1}(V_i-W_i')}}\geq e^{-\frac{r\delta \log n}{2}}\right]\\
&\leq&e^{r\delta\frac{\log n}{2}}(\mathbb{E}[e^{\frac{1}{2}V_1}e^{-\frac{1}{2}W_1}])^{\frac{n^{m-1}}{2^{m-1}(m-1)!}} \\
&=&e^{r\delta\frac{\log n}{2}}[e^{-\frac{1}{2}\log (\frac{p}{q})}\alpha p+e^{-\frac{1}{2}\log (\frac{1-p}{1-q})}\alpha (1-p)+(1-\alpha)]^{\frac{n^{m-1}}{2^{m-1}(m-1)!}} \\
&&\times[e^{\frac{1}{2}\log (\frac{p}{q})}\alpha q+e^{\frac{1}{2}\log (\frac{1-p}{1-q})}\alpha (1-q)+(1-\alpha)]^{\frac{n^{m-1}}{2^{m-1}(m-1)!}}.
\end{eqnarray*}
Taking logrithm on both side yields
\begin{eqnarray*}
\log (I)&\leq&\frac{1}{2}r\delta \log n+\frac{n^{m-1}\alpha}{2^{m-1}(m-1)!}[2\sqrt{pq}+2\sqrt{(1-p)(1-q)}-2] \\
&=&\frac{1}{2}\frac{\log n}{\sqrt{\log (\frac{1}{\delta})}}-\frac{t\log n}{2^{m-1}(m-1)!}[(\sqrt{p}-\sqrt{q})^{2}+(\sqrt{1-p}-\sqrt{1-q})^{2}].
\end{eqnarray*}
Hence,
\[(I)\leq n^{-t\frac{(\sqrt{p}-\sqrt{q})^{2}+(\sqrt{1-p}-\sqrt{1-q})^{2}}{2^{m-1}(m-1)!}(1+o(1))}=n^{-\frac{t}{I_m(p,q)}(1+o(1))}.\]
Since $t>I_m(p,q)$ by assumption, we get $(I)\leq n^{-(1+\epsilon)}$ for some small constant $\epsilon>0$ and hence
\[p_i\leq (I)+(II)\leq n^{-(1+\epsilon)}.\]
By union bound, the probability that there exists a mislabelled node is bounded by $n^{-\epsilon}=o(1)$. Then the proof is complete.

\section{Proof of Theorem \ref{Spectra}}\label{analysis}
Recall $\widetilde{A}_{i_1i_2\dots i_m}=\mathbbm{1}{[A_{i_1i_2\dots i_m}=1]}$. For convenience, let $e=\{i_1,i_2,\dots,i_m\}$ for distinct nodes $ i_1,i_2,\dots,i_m$. Let $p_1=p\alpha,q_1=q\alpha$. Then \\
\begin{equation*}
\begin{split}
\mathbb{E}(\widetilde{A}_e)&=\left\{ 
\begin{array}{cc}
    p_{1} &  e\subset I_+(\sigma),or\ I_-(\sigma);\\
    q_{1}, & otherwise.
\end{array}
\right.
\end{split}
\end{equation*}

\begin{equation*}
\begin{split}
Var(\widetilde{A}_e)
&=\left\{
\begin{array}{cc}
    p_{1}(1-p_1), & e\subset I_+(\sigma),or\ I_-(\sigma); \\
    q_{1}(1-q_1), & otherwise.
\end{array}
\right.
\end{split}
\end{equation*}

Let $\mathbb{I}_{e}$ denote $n$-dimensional vector with $i_l$-th position one, $i_l\in e$, $l=1,2,\dots,m$, and other positions are zero. Denote $\sigma_{e}=diag(\sigma)\cdot\mathbb{I}_{e}$ and
\[L=\sum\limits_{e}\widetilde{A}_e\left[(\mathbb{I}_{e}^{\intercal}\sigma_{e})diag(\sigma_{e})-\sigma_{e}\sigma_{e}^{\intercal}\right].\]
Let $I$ be the identity matrix and $M=I-\frac{II^{T}}{n}-\frac{\sigma\sigma^{T}}{n}.$ By Proposition 2 in \cite{KBG18}, it suffices to show the third smallest eigenvalue of $M(\mathbb{E}(L))M$ is larger than zero with probability $1-o(1)$, that is, $\lambda_{3}\left(M\mathbb{E}(L)M\right)>0$.

Firstly, we have the following result.

\begin{Proposition}
Let $M=I-\frac{II^{T}}{n}-\frac{\sigma\sigma^{T}}{n}.$ Then
\[M(\mathbb{E}(L))M=\frac{p_{1}-q_{1}}{2}n\binom{\frac{n}{2}-2}{m-2}M,\] and \[\lambda_{3}\left(M\mathbb{E}(L)M\right)=\frac{p_{1}-q_{1}}{2^{m-1}}\frac{n^{m-1}}{(m-2)!}\left(1+o(1)\right).\] 
\end{Proposition}
\begin{proof}
Simple calculation yields
\begin{eqnarray*}
\mathbb{E}L&=&\frac{<\mathbb{E}L,M>}{n-2}M+\frac{<\mathbb{E}L,\sigma\sigma^{\intercal}>}{n^{2}}\sigma\sigma^{\intercal} \\
<\mathbb{E}L,\sigma\sigma^{\intercal}>&=&-q_{1}n^{2}\binom{n-2}{m-2} \\ <\mathbb{E}L,M>&=&\frac{p_{1}-q_{1}}{2}n\binom{n}{n-2}\binom{n/2-2}{m-2}. \\
\end{eqnarray*}
Hence, $M(\mathbb{E}(L))M=\frac{<\mathbb{E}L,M>}{n-2}M$. The proof is complete.

\end{proof}

Next we present the Matrix Bernstein inequality.
\begin{Lemma}[Matrix Bernstein inequality]\label{mbernstein}
Let $\{X_{k}\}$ be a finite sequence of independent, symmetric random matrices of dimension N. Suppose that $\mathbb{E}X_{k}=0$ and $\|X_{k}\|\leq M$ almost surely for all k. Then for all $x\geq 0,$
\[\mathbb{P}(\|\sum\limits_{k}X_{k}\|\geq x)\leq N\cdot \exp\left(-\frac{\frac{x^{2}}{2}}{v^{2}+\frac{Mx}{3}}\right),\]     where $v^{2}=\|\sum\limits_{k}\mathbb{E}X_{k}^{2}\|.$
\end{Lemma}

\textup{Recall that}
\begin{center}
$L=\sum\limits_{e}\widetilde{A}_{e}((\mathbb{I}_{e}^{\intercal}\sigma_{e})diag(\sigma_{e})-\sigma_{e}\sigma_{e}^{\intercal}).$
\end{center}
\textup{Hence,}
\begin{center}
$\Pi(L-\mathbb{E}L)\Pi=\sum\limits_{e}(\widetilde{A}_{e}-\mathbb{E}(\widetilde{A}_{e})\cdot \Pi((\mathbb{I}_{e}^{\intercal}\sigma_{e})diag(\sigma_{e})-\sigma_{e}\sigma_{e}^{\intercal})\Pi.$
\end{center}
We note that
\begin{center}
$\|\Pi((\mathbb{I}_{e}^{\intercal}\sigma_{e})diag(\sigma_{e})-\sigma_{e}\sigma_{e}^{\intercal})\Pi\|\leq |\mathbb{I}_{e}^{\intercal}\sigma_{e}|+\|\sigma_{e}\|^{2}\leq 2m$
\end{center}
for any hyperdge $e$. By Lemma \ref{mbernstein}, we have
\begin{center}
$\mathbb{P}(\|\Pi(L-\mathbb{E}L)\Pi\|\geq x)\leq n\cdot \exp\left(-\frac{\frac{x^{2}}{2}}{v^{2}+\frac{2mx}{3}}\right)$
\end{center}
\textup{where}
\begin{center}
$v^{2}=\left\|\sum\limits_{e}\mathbb{E}((A_{\mathcal{H}})_{e}-\mathbb{E}(A_{\mathcal{H}})_{e})^{2}(\cdot \Pi((\mathbb{I}_{e}^{\intercal}\sigma_{e})diag(\sigma_{e})-\sigma_{e}\sigma_{e}^{\intercal})\Pi)^{2}\right\|.$
\end{center}

Let $v^{2}=\|\sum\limits_{e}\mathbb{E}(\widetilde{A}_e-\mathbb{E}(\widetilde{A}_e))^{2}[M((\mathbb{I}_{e}^{\intercal}\sigma_{e})diag(\sigma_{e})-\sigma_{e}\sigma_{e}^{\intercal})M]^{2}\|$, $V=\sum\limits_{e}\mathbb{E}(\widetilde{A}_e-\mathbb{E}(\widetilde{A}_e))^{2}Y_{e}$ and \[Y_{e}=[(\mathbb{I}_{e}^{\intercal}\sigma_{e})diag(\sigma_{e})-\sigma_{e}\sigma_{e}^{\intercal})]M[(\mathbb{I}_{e}^{\intercal}\sigma_{e})diag(\sigma_{e})-\sigma_{e}\sigma_{e}^{\intercal})].\]
The following result gives the expression of $V$ and $v^2$.
\begin{Proposition}
$V=t_{1}\frac{1}{n}\sigma\sigma^{\intercal}+t_{2}M$, where \[t_{1}=\frac{mn^{m-1}}{(m-2)!}q_{2},\ \ \ \  t_{2}=\frac{n^{m-1}}{(m-2)!}[\frac{m}{2^{m-1}}p_{2}-\frac{m-2^{m}}{2^{m-1}}q_{2}],\] \\
and $v^{2}=\|M VM\|=\|t_{2}M\|=t_{2}(1+o(1))$.
\end{Proposition}
\begin{proof}
Note that
$v^{2}=\|\sum\limits_{e}\mathbb{E}(\widetilde{A}_e-\mathbb{E}(\widetilde{A}_e))^{2}M Y_{e} M\|=\|M VM\|$, 
we need to show that $V=t_{1}\frac{1}{n}\sigma\sigma^{\intercal}+t_{2}M$.
Note that $V=\sum\limits_{e\in\mathbbm{1}_{+}(\sigma),\mathbbm{1}_{-}(\sigma)}p_{1}(1-p_{1})Y_{e}+\sum\limits_{\substack{e\not\in\mathbbm{1}_{+} \\ e\not\in \mathbbm{1}_{-}}}q_{1}(1-q_{1})Y_{e}.$ 
It is clear that 
\[V=\frac{<V,M>}{<M,M>}M+<V,\frac{1}{n}\sigma\sigma^{\intercal}>\frac{1}{n}\sigma\sigma^{\intercal}.\] 
Hence, $t_{1}=\frac{1}{n}<V,\sigma\sigma^{\intercal}>, t_{2}=\frac{1}{n-2}<V,M>=\frac{1}{n-2}[tr(V)-\frac{\sigma M\sigma^{\intercal}}{n}]$. Direct calculation yields \[t_{1}=\frac{mn^{m-1}}{(m-2)!}q_{2},\ \ \  t_{2}=\frac{n^{m-1}}{(m-2)!}\left[\frac{m}{2^{m-1}}p_{2}-\frac{m-2^{m}}{2^{m-1}}q_{2}\right].\] 
\end{proof}

\noindent
\textit{Proof of Theorem \ref{Spectra}.} Note that 
\[M(L-\mathbb{E}(L))M=\sum\limits_{e}[(\widetilde{A}e-\mathbb{E}(\widetilde{A}e))]M((\mathbb{I}_{e}^{\intercal}\sigma_{e})diag(\sigma_{e})-\sigma_{e}\sigma_{e}^{\intercal}))M.\] 
By Lemma \ref{mbernstein}, let $x=(1+\epsilon)\sqrt{2\log n}v,$ 
\begin{eqnarray*}
\mathbb{P}(\|M(L-\mathbb{E}(L)M\|)\geq t)
&\leq& ne^{-\frac{\frac{(1+L)^{2}2\log n      }{2}v^{2}}{v^{2}+\frac{4m(1+\epsilon)\sqrt{2\log n}v}{3}}}=ne^{-\frac{(1+\epsilon)^{2}\log n}{1+\frac{4m(1+\epsilon)\sqrt{2\log n}}{3v}}} \\
&\leq& e^{-\log n\left\{\frac{(1+\epsilon)^{2}}{1+\frac{4m(1+\epsilon)\sqrt{2\log n}}{3v}}-1\right\}}.
\end{eqnarray*}
If $\frac{4m\sqrt{2\log n}}{3v}\leq 1$, that is, 
\begin{equation}\label{tgreat}
t\geq\frac{32m^2(m-2)!2^{m-1}}{9(mp-(m-2^m)q)},
\end{equation}
take $\epsilon=1$. If $\frac{4m\sqrt{2\log n}}{3v}>1$, 
take $1+\epsilon=\delta\frac{8m\sqrt{2\log n}}{3v}$ for any constant $\delta>1$.
Either case, we have \[\mathbb{P}(\|M(L-\mathbb{E}(L)M\|)\geq t)=o(1).\]
With probability $1-o(1)$, one has
\begin{center}
$\|M(L-\mathbb{E}(L)M\|<(1+\epsilon)\sqrt{2t_{2}\log n}$.
\end{center}
If
\begin{equation}\label{pq1}
\frac{(p_{1}-q_{1})}{2^{m-1}}\frac{n^{m-1}}{(m-2)!}> (1+\epsilon)\sqrt{2t_{2}\log n},
\end{equation}
then
$\lambda_{3}(M\mathbb{E}(L)M)>0$ with probability $1-o(1).$

When (\ref{tgreat}) holds, $\epsilon=1$ and (\ref{pq1})
is equivalent to
\begin{eqnarray*}
[\frac{(p_{1}-q_{1})}{2^{m-1}}\frac{n^{m-1}}{(m-2)!}]^{2}&>&(1+\epsilon)^{2}2\log n \left(\frac{n^{m-1}}{(m-2)!}[\frac{m}{2^{m-1}}p_{2}-\frac{m-2^{m}}{2^{m-1}}q_{2}]\right), \\
\frac{(p_{1}-q_{1})}{2^{m-1}}\frac{n^{m-1}}{(m-2)!}&>&(1+\epsilon)^{2}2\log n [mp_{2}-(m-2^{m})q_{2}], \\
\frac{n^{m-1}(p-q)^{2}\alpha^{2}(1-\eta)^{2}}{2^{m-1}(m-2)!}&>&(1+\epsilon)^{2}2\log n [m\alpha p-(m-2^{m})\alpha q],\\
\frac{n^{m-1}(p-q)^{2}\alpha}{2^{m-1}(m-2)!}&>&(1+\epsilon)^{2}2\log n [mp-(m-2^{m})q].
\end{eqnarray*}
Note that $\alpha=\frac{t\log n}{n^{m-1}}$. Hence
\begin{eqnarray*}
\frac{t(p-q)^{2}}{2^{m-1}(m-2)!}&>&8[mp-(m-2^{m})q].
\end{eqnarray*}
Then
\[t>8\frac{2^{m-1}(m-2)![mp-(m-2^{m})q] }{(p-q)^2}=J_m(p,q),\]
which is larger than the right hand of (\ref{tgreat}).

When (\ref{tgreat}) fails, $1+\epsilon=\delta\frac{8m\sqrt{2\log n}}{3v}$ and (\ref{pq1})
is equivalent to

\begin{equation}\label{pq1}
\frac{(p_{1}-q_{1})}{2^{m-1}}\frac{n^{m-1}}{(m-2)!}> \delta\frac{8m\sqrt{2\log n}}{3}\sqrt{2\log n},
\end{equation}
hence
\[t>\delta\frac{16m}{3}\frac{2^{m-1}(m-2)!}{p-q},\ \ \delta>1,\]
which is not possible. Then the proof is complete.


\begin{thebibliography}{9}

\bibitem{A18}
Abbe, E., Community Detection and Stochastic Block Models:
Recent Developments. {\em Journal of Machine Learning Research} {\bf 2018}, {\em 18}, 1-86.

\bibitem{ABH16}
Abbe, E, Bandeira, A.S. and Hall, G. (2016).
Exact recovery in the stochastic block model.
\textit{IEEE Transactions on Information Theory}, 62(1): 471-487.



\bibitem{ABBS14}
Abbe, E., Bandeira,A.S., Bracher, A. and Singer, A.(2014). Decoding binary node labels from censored edge measurements: phase transition and efficient recovery.
\textit{IEEE Transactions on Network Science and Engineering}, 1(1), 10-22.



\bibitem{ALS18}
Ahn, K., Lee, K. and Suh, C. (2018).
Hypergraph Spectral Clustering in the Weighted Stochastic Block Model.
\textit{IEEE Journal of Selected Topics in Signal Processing}, 12(5), 2018.

\bibitem{ALS19}
Ahn, K., Lee, K. and Suh, C. (2019).
Community Recovery in Hypergraphs.
\textit{IEEE Transactions on Information Theory}, 12(5), 6561-6578.






%\bibitem{AC09}
%Anderson, R. and Chellapilla, K. 2009.
%Finding dense subgraphs with size bounds.
%\textit{WAW 2009}, 25-37.

\bibitem{BTYZQ21}
Bi,X., Tang, X., Yuan, Y., Zhang, Y. and Qu,A.(2021).
Tensors in statistics.
\textit{Annual Review of Statistics and Its Application},8,345-368.


%\bibitem{BGK17}
%Bhangale, A.,  Gandhi, R. and Kortsarz, G. (2017). Improved approximation algorithm for the
%dense-3-subhypergraph problem. https://arxiv.org/abs/1704.08620




%\bibitem{C00}
 %Charikar, M. 2000. Greedy approximation algorithms for finding dense components in a graph. 
 %\textit{APPROX}, 84–95.


%\bibitem{CKKMP18}
%Chiplunkar, A., Kapralov, M., Khanna, S., Mousavifar, A. and Peres, Y.(2018). Testing graph clusterability: algorithms and lower bounds. \textit{2018 IEEE 59th Annual Symposium on Foundations of Computer Science}, 497-508.




 
 \bibitem{CLW18}
Chien, I., Lin, C. and Wang, I.(2018).
Community Detection in Hypergraphs:
Optimal Statistical Limit and Efficient Algorithms.
\textit{Proceedings of the Twenty-First International Conference on Artificial Intelligence and Statistics}, 84:871-879.


%\bibitem{CY06}Chen, J. and  B. Yuan. 2006.
%Detecting functional modules in the yeast proteinprotein interaction
%network. \textit{Bioinformatics}, \textbf{22(18)}, 2283-2290.


%\bibitem{CS10}
%Chen, J. and Y. Saad. 2010.
%Dense subgraph extraction with application to community detection,
%\textit{IEEE Transactions on Knowledge and Data Engineering} 24,7: 1216-1230.

\bibitem{CY06}Chen, J. and Yuan, B. (2006).
Detecting functional modules in the yeast proteinprotein interaction
network. \textit{Bioinformatics}, \textbf{22(18)}, 2283-2290.




%\bibitem{CK10}Chertok, M. and Keller, Y. (2010).
%Efficient high order matching. 
%\textit{IEEE
%Transactions on Pattern Analysis and Machine Intelligence}, \textbf{32(12)}, 2205-2215.

%\bibitem{CDKKR18}
%Chlamtac, E., M. Dinitz, C. Konrad, G. Kortsarz, and G. Rabanca. %2018. 
%The Densest k-Subhypergraph problem. 
%\textit{SIAM Journal on Discrete Mathematics}. 32 (2):1458–77.






\bibitem{DGMS21}
Dhara, S., Gaudio, J., Mossel, E. and Sandon, C.(2021).
Spectral recovery of binary cencored block models.
\url{https://arxiv.org/pdf/2107.06338.pdf}




%\bibitem{Fort} Fortunato,S. (2010). Community detection in graphs. \textit{Physics Reports}, \textbf{486 (3-5)},
%75-174.










%\bibitem{GKT05}
%Gibson, D., R. Kumar, and A. Tomkins. 2005.
%Discovering large dense subgraphs in massive
%graphs.
%\textit{VLDB 05},721–732.



\bibitem{GD14} Ghoshdastidar, D. and Dukkipati, A. 
Consistency of spectral partitioning of
uniform hypergraphs under planted partition model. \textit{ Advances in Neural Information Processing Systems (NIPS)} {\bf 2014}, 397-405.




\bibitem{GD17} Ghoshdastidar, D. and Dukkipati A. 
Consistency of spectral hypergraph partitioning under planted partition model. \textit{The Annals of Statistics} {\bf 2017},
\textbf{45(1)}, 289-315.

\bibitem{GH16}
Gile, K and Handcock M. (2016).
Analysis of networks with missing data with application to the National Longitudinal Study of Adolescent Health,
\textit{Journal of the Royal Statistical Society Series C Applied Statistics}, 66, 501-519.


\bibitem{GMZZ16}Gao, C., Ma, Z., Zhang, A.Y. and Zhou, H. H. (2016).
Community detection in degree-corrected block models. \url{https://arxiv.org/pdf/1607.06993.pdf}.

%\bibitem{HWX17}
%Hajek, B., Wu, Y. and Xu, J.(2017).  Information limits for recovering a hidden community. 
%\textit{IEEE Transaction on Information Theory}, 63(8): 4729-4745.

\bibitem{HWX16}
Hajek, B., Wu, Y. and Xu, J. (2016).
Achieving exact cluster recovery threhold via semidefinite programming. 
\textit{IEEE Transaction on Information Theory}, 62(5): 2788-2796.



\bibitem{HWX18}
Hajek, B., Wu, Y. and Xu, J. 2018.
Recovering a hidden community beyond the Kesten Stigum threshold in $O(|E|\log^{\*}|V|)$ time. 
\textit{Journal of Applied Probability}, 55, 2: 325-352.

\bibitem{J15}
Jin,J. (2015).
Fast community detection by SCORE. 
\textit{The Annals of Statistics}, 43, 1: 57-89.

%\bibitem{HSBSSF16}
%Hooi,B., H.A., Song, A. Beutel, N. Shah, K. Shin,and C. Faloutsos. 2016.
%raudar:bounding graph fraud in the face of camouflage,
%\textit{KDD,ACM},895-904.


%\bibitem{HWC17}
%Hu, S., Wu, X. and Chan, T-H. (2017).
%Maintaining densest subsets efficiently in evolving hypergraphs
%\textit{Proceedings of the 2017 ACM on Conference on Information and Knowledge Management},929-938,Singapore.

\bibitem{H09}
Huisman, M. (2009).
Imputation of missing network data: Some simple procedures. \textit{Journal of Social Structure}, 10,1-29.



%\bibitem{FL18}
%Friedland, S. and Lim, L. 2018.
%Nuclear norm of higher-order tensors. \textit{Mathematics of Computation}, 87(311), 1255-1281.

\bibitem{GZFA} Goldenberg, A.,   Zheng, A. X. S.,  Fienberg, E., and Airoldi, E. M. (2010).
A survey of statistical network models. 
\textit{Foundations and Trends in Machine Learning 2}, \textbf{2}, 129-233.


%\bibitem{KS09}
%Khuller, S. and Saha, B. 2009.
%On finding dense subgraphs,
%\textit{ICALP 2009}: 597-608. 


\bibitem{KSX20}
Ke, Z., Shi, F. and Xia, D. (2020). Community detection for hypergraph networks via
regularized tensor power iteration. {\bf 2020}, https://arxiv.org/pdf/1909.06503.pdf.


\bibitem{K11}
Kim, S. 2011. Higher-order correlation clustering for image segmentation. 
\textit{Advances in Neural
Information Processing Systems} 1530–8.


\bibitem{KBG18}
Kim, C., Bandeira, A. and Goemans, M.(2018). Stochastic block model for hypergraphs: statistical limits and a semidefinite programming approach.


\bibitem{LR15}Lei, J. and Rinaldo, A. (2015).
Consistency of spectral clustering in stochastic block models. \textit{The Annals of Statistics}, \textbf{43(1)}, 215-237.


\bibitem{LGYS18}
Liu, M., Gao, Y., Yap, P. and Shen, D.(2018).
Multi-Hypergraph learning for incomplete
multimodality data,
\textit{IEEE JOURNAL OF BIOMEDICAL AND HEALTH INFORMATICS}, 22,1197-1208.


%\bibitem{LKH21}
%Liang, J., Ke C. and Honorio, J.(2021). Information theoretic limits of exacat recovery in subhypergraph models for community detection. \url{https://arxiv.org/pdf/2101.12369.pdf}





\bibitem{LZ20}
Luo, Y. and Zhang, A. (2020).
Open problem: average-case hardness of hypergraphic planted clique detection.
\textit{Proceedings of Machine Learning Research}, 1-4, 2020.



\bibitem{LJY15}
Liu, H., Jan, L. and Yan, S.(2015).
Dense subgraph partition of positive hypergraphs.
\textit{IEEE Transactions on Pattern Analysis and Machine Intelligence},37,3:541-554.

\bibitem{MNS15} Mossel, E., Neeman, J. and Sly, A. (2015).
Reconstruction and estimation in the planted partition model.
\textit{Probability Theory and Related Fields}, \textbf{162}, 431-461.

\bibitem{MNS17} Mossel, E., Neeman, J. and Sly, A. (2017). A proof of the block model threshold conjecture.
\textit{Combinatorica}, 1-44. 


\bibitem{SMM18}
Smith, J., Moody, J. and Morgan, J.(2018).
Network sampling coverage II: The effect of non-random missing data on network measurement,
\textit{Soc Networks}, 48, 78-99.


%\bibitem{SN18}
%Saad, H. and Nosratinia, A.(2018).
%Community detection with side information: exact recovery under the stochastic block model. \textit{IEEE Journal of Selected Topics in Signal Processing}, 12: 944-958.


\bibitem{SM97}
 Shi, J. and Malik, J. (1997).
 Normalized cuts and image segmentation.
 \textit{IEEE Transactions on Pattern Analysis and Machine Intelligence}, \textbf{22(8)}, 888-905.


%\bibitem{T15}
%Tsourakakis, C.(2015).
%The K-clique densest subgraph problem,
%\textit{Proceedings of the 24th International Conference on World Wide Web},1122-1132,Florence, Italy.

\bibitem{WF21}
Weng, H. and Feng, Y.(2021).
Community detection with nodal information: likelihood and its variational approximation.
\textit{Stat}, e428.

%\bibitem{WLKN09}
%Wu, M.,  Li, X.,  Kwoh, C.K. and  Ng. S.K. 2009.
%A coreattachment
%based method to detect protein complexes in
%ppi networks, 
%\textit{BMC bioinformatics}, 10: 169.


\bibitem{YB21}
Yuan,Y. and Qu, A.(2021).
Community detection with dependent connectivity.
\textit{The Annals of Statistics}, 49, 2378-2428.



\bibitem{YS21}
Yuan, M. and Shang, Z. (2021). Informatin Limits for Detection a Subhypergraph. \textit{STAT}, e407.

\bibitem{YS21b}
Yuan, M. and Shang, Z.  Sharp detection boundaries on testing dense subhypergraph. \textit{Bernoulli} {\bf 2021}, to appear.


%\bibitem{ZHS06}
%Zhou, D., Huang,J.  and  Scholkopf, B. (2006). Learning with hyper graphs: Clustering, classification, and embedding. 
%\textit{Advances in Neural Information Processing Systems}. 6:1601–8.

\bibitem{ZLZ11} Zhao, Y.,  Levina, E. and  Zhu., J.(2011).
Community extraction for social networks.
\textit{Proc. Natn. Acad. Sci. USA}, \textbf{108}, 7321-7326.

\bibitem{ZW21}
Zhen, Y. and Wang, J.(2021).
Community detection in general hypergraph via
graph embedding.
\textit{https://arxiv.org/pdf/2103.15035.pdf}.

\end{thebibliography}
\end{document}